\documentclass[11pt]{article}
\usepackage[dvips,letterpaper,margin=1.0in]{geometry}
\usepackage[utf8]{inputenc}
\usepackage{preamble}
\usepackage{fullpage}

\newcommand{\Hm}{\mathbb{H}}

\newcommand{\bsy}{\boldsymbol}

\newcommand{\mY}{\mathbf{Y}}

\newcommand{\mI}{\mathbf{I}}
\newcommand{\mM}{\mathbf{M}}
\newcommand{\mP}{\mathbf{P}}
\newcommand{\mR}{\mathbf{R}}
\newcommand{\mO}{\mathbf{O}}
\newcommand{\mT}{\mathbf{T}}
\newcommand{\mU}{\mathbf{U}}
\newcommand{\mV}{\mathbf{V}}
\newcommand{\mW}{\mathbf{W}}
\newcommand{\diag}{\textsf{diag}}

\renewcommand{\vec}{\textsf{vec}}
\newcommand{\msig}{\mathbf{\Sigma}}
\newcommand{\round}{\textsf{round}}

\title{Model Stealing for Any Low-Rank Language Model}
\author{Allen Liu \thanks{Email: \texttt{cliu568@mit.edu} This work was supported in part by an NSF Graduate Research Fellowship, a Hertz Fellowship, and a Citadel GQS Fellowship.} \and Ankur Moitra \thanks{Email: \texttt{moitra@mit.edu} This work was supported in part by a Microsoft Trustworthy AI Grant, an ONR grant, and a David and Lucile Packard Fellowship.}}
\date{\today}

\begin{document}

\maketitle
\thispagestyle{empty}

\begin{abstract}

Model stealing, where a learner tries to recover an unknown model via carefully chosen queries, is a critical problem in machine learning, as it threatens the security of proprietary models and the privacy of data they are trained on.  In recent years, there has been particular interest in stealing large language models (LLMs).  In this paper, we aim to build a theoretical understanding of stealing language models by studying a simple and mathematically tractable setting.  We study model stealing for Hidden Markov Models (HMMs), and more generally low-rank language models.  

We assume that the learner works in the conditional query model, introduced by Kakade, Krishnamurthy, Mahajan and Zhang \cite{kakade2024learning}. Our main result is an efficient algorithm in the conditional query model, for learning any low-rank distribution.  In other words, our algorithm succeeds at stealing any language model whose output distribution is low-rank. This improves upon the result in  \cite{kakade2024learning} which also requires the unknown distribution to have high ``fidelity" \--- a property that holds only in restricted cases.  There are two key insights behind our algorithm: First, we represent the conditional distributions at each timestep by constructing barycentric spanners among a collection of vectors of exponentially large dimension. Second, for sampling from our representation, we iteratively solve a sequence of convex optimization problems that involve projection in relative entropy to prevent compounding of errors over the length of the sequence. This is an interesting example where, at least theoretically, allowing a machine learning model to solve more complex problems at inference time can lead to drastic improvements in its performance. 


\end{abstract}

\newpage

\setcounter{page}{1}

\section{Introduction}

Proprietary machine learning models are often highly confidential. Not only are their weights not publicly released, but even their architecture and hyperparameters used in training are kept a closely guarded secret. And yet these models are often deployed as a service, allowing users to make queries to the model and receive answers. These answers can take the form of labels or completions of prompts, and sometimes a model will even report additional information such as its confidence scores. This raises a natural question:

\begin{aquestion}
    Are these black-box models actually secure, or is it possible to reverse engineer their parameters or replicate their functionality just from query access to them?
\end{aquestion}

\noindent This task is called {\em model stealing} and it threatens the security of proprietary models and the privacy of data they are trained on. Beyond nefarious reasons, it can also be used in {\em model distillation} \cite{ma2021undistillable}, where we have trained a large and highly complex model and we want to transfer its knowledge to a much smaller model. It can also be a useful tool for identifying vulnerabilities, as those are often inherited by stolen models. 

In any case, model stealing continues to be a very active area of research. The influential work in \cite{tramer2016stealing} showed that there are simple and efficient attacks on popular models like logistic regression, decision trees and deep neural networks that often work in practice. Since then, many new attacks and defenses have been formulated \cite{he2021model, he2021stealing, reith2019efficiently, wang2018stealing, juuti2019prada, oliynyk2023know, orekondy2019prediction, wang2020information}. There are also approaches based on embedding watermarks \cite{jia2021entangled, zhao2023protecting, li2022defending} that make it possible to detect when one model has been stolen from another. In recent years, there has been particular interest in stealing large language models (LLMs). Various works have shown how to steal isolated components of a language model such as the decoding algorithm \cite{naseh2023stealing}, prompts used to fine-tune the model \cite{sha2024prompt}, and even the entire weight matrix of the last layer (the embedding projection layer) \cite{carlini2024stealing}.

In this work, our main interest will be in theoretical foundations for stealing language models. As is all too familiar, proving rigorous end-to-end guarantees when working with modern machine learning models with all their bells and whistles seems to be an extremely difficult task. For example, while we can understand the training dynamics on multilayer neural networks in terms of gradient flow in the Wasserstein space of probability distributions \cite{mei2019mean, nguyen2023rigorous}, it has turned out to be quite difficult to analyze these dynamics except in simplified settings with a high degree of symmetry. Even worse, there are strong lower bounds for learning deep neural networks \cite{klivans2009cryptographic} even with respect to nice input distributions \cite{goel2020superpolynomial, diakonikolas2020algorithms, goel2020statistical}. The task of reasoning about modern language models seems no easier, as they are built on transformers \cite{attention} with many building blocks such as word embeddings, positional encodings, queries, keys, values, attention, masking, feed-forward neural networks and layer normalization. 

Nevertheless there are often simplified models that abstract important features of more complex models and give us a sandbox in which to try to find theoretical explanations of empirical phenomena. For example, analyzing the dynamics of gradient descent when training a deep neural network is notoriously difficult. But in an appropriate scaling limit, and when the network is wide enough, it can be approximated through the neural tangent kernel \cite{jacot2018neural}. For recurrent neural networks, a popular approach is to analyze gradient descent on linear dynamical systems instead \cite{hardt2018gradient}. Likewise for language models, it is natural to work with Hidden Markov Models (HMMs), which are in some sense the original language model, dating back to the work of Claude Shannon in 1951 \cite{shannon1951prediction} and were the basis of other early natural language processing systems including the IBM alignment models. More broadly, we can consider a generalization called low-rank language models (Definition~\ref{def:low-rank}). This brings us to our main questions:

\begin{aquestion}
    Is there an efficient algorithm for stealing HMMs from query access? What about more generally for low-rank language models?
\end{aquestion}

\noindent These questions were first introduced and studied in an exciting recent work of Kakade, Krishnamurthy, Mahajan and Zhang \cite{kakade2024learning}. However their motivation and framing was somewhat different, as we will explain.

\subsection{Main Results}

Formally, we view a language model as a distribution $\Hm$ over $\calO^T$ for some alphabet $\calO$ and sequence length $T$.  For simplicity, we treat the sequence length as fixed.  Following Kakade, Krishnamurthy, Mahajan and Zhang \cite{kakade2024learning}, the rank of the distribution generated by a language model is defined as follows:

\begin{definition}\label{def:low-rank}[Low Rank Distribution]
A distribution $\Hm$ over $\calO^T$ for alphabet $\calO$ of size $O$ and sequence length $T$ is rank $S$ if for all $t < T$, the $O^{T - t} \times O^t$ matrix, $\mM^{(t)}$,  with entries equal to $\Pr_{\Hm}[f|h]$ for $f \in \calO^{T - t}$ and $h \in \calO^t$ has rank at most $S$.
\end{definition}

\noindent In other words, a distribution is rank-$S$ if for any $t < T$, the information in the prefix of length $t$ can be embedded in  an $S$-dimensional space such that the distribution of the future tokens can be represented as a linear function of this embedding.  We note that low-rank distributions are expressive and encompass distributions generated by a Hidden Markov Model (HMM) with $S$ hidden states (see Fact~\ref{fact:low-rank}). 

Next, we formalize the setup for studying model stealing.  We allow the learner to make {\em conditional queries} \--- that is, the learner can specify a history of observations, and then receives a random sample from the conditional distribution on the future observations.  Formally, we have the following definition: 

\begin{definition}[Conditional Query]
The learner may make conditional queries to a distribution $\Hm$ by querying a string $h \in \calO^t$ where $0 \leq t < T$.  Upon making this query, the learner obtains a string $f$ of length $T - t$ drawn from the distribution $\Pr_{\Hm}[f|h]$.
\end{definition}

In this model, our goal is to design an algorithm that makes a total  number of queries that is polynomial in $S, O , T$ and learns an efficiently samplable distribution that is $\epsilon$-close in total variation distance to $\Hm$. This conditional query model was recently introduced by Kakade, Krishnamurthy, Mahajan and Zhang \cite{kakade2024learning}. Their motivation was two-fold: First, while learning an HMM from random samples is known to be computationally hard \cite{mossel2005learning}, in principle one can circumvent these barriers if we are allowed conditional samples. Second, a solution to their problem would generalize Angluin's classic $L^*$ algorithm which learns deterministic finite automata from membership queries \cite{angluin1987learning}. In terms of results, Kakade, Krishnamurthy, Mahajan and Zhang \cite{kakade2024learning} introduced a notion that they called {\em fidelity} and gave a polynomial time algorithm to learn any low-rank distribution (and thus any HMM) which has high fidelity through conditional queries. However this property does not always hold. Thus, their main question still remains: Is it possible to learn arbitrary low-rank distributions through conditional queries? Here we resolve this question in the affirmative. We show:

\begin{theorem}\label{thm:full-learning}
Assume we are given conditional query access to an unknown rank $S$ distribution $\Hm$ over $\calO^T$ where $|\calO| = O$.  Then given a parameter $0 < \eta < 1$, there is an algorithm that takes $\poly(S,O, T, 1/\eta)$ conditional queries and running time and with probability $1 - \eta$ outputs a description of a distribution $\Hm'$ such that $\Hm'$ is $\eta$-close in TV distance to $\Hm$. Moreover there is an algorithm that samples from $\Hm'$ in $\poly(S,O, T, \log(1/\eta))$ time.
\end{theorem}

Note that crucially, the algorithm only makes conditional queries for learning the representation of $\Hm'$.  Once we have this learned representation, we may draw as many samples as we want without making any more queries to the original distribution $\Hm$.

\subsection{Discussion} 

Theorem~\ref{thm:full-learning} shows that we can efficiently learn \emph{any} low-rank distribution via conditional queries. Thus, we can view our results as showing that in some sense, the rank of a distribution can be a useful proxy for understanding the complexity of model stealing, similar to how complexity measures, such as Bellman rank \cite{jiang2017contextual} and its relatives \cite{foster2021statistical}, are useful for understanding the statistical complexity of learning a near optimal policy in reinforcement learning.

There is a key conceptual insight driving our algorithm. One of the challenges in learning sequential distributions is that the error can grow exponentially in the sequence length $T$. In particular if we imagine sampling from $\Hm$ one token at a time,  the low-rank structure ensures that we only need to keep track of an $S$-dimensional hidden state at each step.  However, each step involves multiplication by a change-of-basis matrix and these repeated multiplications cause the error to grow exponentially.  The key to mitigating this error blowup is that we combine each change-of-basis with a projection step, where we solve a convex optimization problem that performs a projection with respect to KL divergence.  Crucially, projection in KL divergence has a contractive property (see Fact~\ref{fact:kl-basic-intro}) that does not hold for other natural measures of distance between distributions, such as TV distance.  We give a more detailed overview of our algorithm in Section~\ref{sec:tech-overview}. This is an interesting example where allowing a machine learning model to solve more complex problems at inference time can lead to drastic improvements in its performance. Of course, phrased in general terms, this is a driving philosophy behind OpenAI's o1 model. But so far we have little theoretical understanding of the provable benefits of allowing more compute at inference time.  

\subsection{Related Work}
There has been a long line of work on learning HMMs from random samples. Mossel and Roch \cite{mossel2005learning} gave the first polynomial time algorithms that work when under appropriate full rankness conditions on the transition and observation matrices. Other works gave spectral \cite{hsu2012spectral} and method of moments based approaches \cite{anandkumar2012method}. Learning HMMs can also be thought of as a special case of learning phylogenetic trees \cite{cryan2001evolutionary, mossel2005learning}. Other approaches assume the output distributions belong to a parametric family \cite{kontorovich2013learning}, or study quasi-HMMs \cite{huang2015minimal}. 

The main computational obstruction in this area is that HMMs, without full rankness conditions, can encode noisy parities \cite{mossel2005learning}, which are believed to be computationally hard to learn. In the overcomplete setting, where the hidden state space is allowed to be larger than the observation space, there are ways around these lower bounds. First, one can aim to predict the sequence of observations, rather than learning the parameters \cite{sharan2018prediction}. Under a natural condition called multi-step observability one can get quasi-polynomial time algorithms \cite{golowich2023planning}. Second, one can make structural assumptions that the  transition matrices are sparse, well-conditioned, and have small probability mass on short cycles \cite{sharan2017learning}. Alternatively there are polynomial time algorithms that work under the assumption that the transition and observation matrices are smoothed \cite{bhaskara2019smoothed}. 

There are also related models called linear dynamical systems where the hidden states and observations are represented as vectors. In contrast, in HMMs the hidden states and observations take on only a finite set of possible values. There is a long line of work on learning linear dynamical systems too \cite{hardt2018gradient, oymak2019non, tsiamis2019finite, sarkar2019nonparametric, simchowitz2019learning, bakshi2023new, chen2022learning, bakshi2023tensor}.  However, these algorithms require various structural assumptions, and even the weakest ones, namely condition-number bounds on the called observability and controllability matrices, are known to be necessary \cite{bakshi2023new}. 



The conditional query model has also been studied for other statistical problems such as testing discrete distributions and learning juntas  \cite{chakraborty2013power, canonne2015testing, bhattacharyya2018property, chen2021learning}.  However, in most of these settings, the goal of studying the conditional query model is to obtain improved statistical rates rather than sidestep computational hardness. We remark that there are also many classic examples of problems in computational learning theory where, when allowed to make queries, there are better algorithms \cite{jackson1997efficient} than if we are only given passive samples.


\section{Basic Setup}



\noindent Recall that we have conditional query access to some unknown distribution $\Hm$ over $\calO^T$ for some alphabet $\calO$.  We also assume that $\Hm$ has rank $S$ (recall Definition~\ref{def:low-rank}).  Our goal will be to learn a description of $\Hm$ via conditional queries. Let us first formally define Hidden Markov Models (HMMs): 

\begin{definition}[Hidden Markov Model]
An HMM $\Hm$ with  state space $\calS$, $|\calS| = S$, and observation space $\calO$, $|\calO| = O$, is specified by an initial distribution $\mu$ over the states, a transition matrix $\mT \in \R^{S \times S}$, and an emission matrix $\mO \in \R^{O \times S}$.  For a given sequence length $T$, the probability of generating a given sequence $x_1, \dots , x_T$ with each $x_i \in \calO$ is
\[
\Pr[x_1, \dots, x_T] = \sum_{s_1, \dots , s_{T+1} \in \calS^{T + 1}} \mu(s_1) \prod_{t=1}^T  \mO_{x_t,s_t} \mT_{s_{t+1},s_t} \,.
\]
\end{definition}

Next we give a basic observation (see \cite{kakade2024learning}) that the class of rank $S$ distributions contains the set of all distributions generated by an HMM with $S$ hidden states.


\begin{fact}\cite{kakade2024learning}\label{fact:low-rank}
Let $\Hm$ be a distribution over $\calO^T$ generated by an HMM with $S$ hidden states.  Then $\Hm$ is a rank at most $S$ distribution.
\end{fact}

\paragraph{Notation} Throughout this paper, we will use the same global notation for the unknown distribution $\Hm$ with rank $S$, observation space $\calO$ of size $O$ and sequence length $T$.   
We also write $\Pr_{\Hm}[f|h]$ for sequences $f$ of length less than $T - t$.  This will be used to denote the probability that conditioned on the prefix $h$, the observations immediately following $h$ are those in $f$.


 We use the following notation for prefixes and conditional probabilities. For a distribution $\Hm$ on $\calO^T$ and $t < T$, we use $\Hm[:t]$ to denote the distribution induced by $\Hm$ on $t$-character prefixes.  For any prefix $h \in \calO^t$, we use $\Pr_{\Hm}[\cdot | h]$ to denote the distribution on futures $f \in \calO^{T - t}$ given by $\Pr_{\Hm}[f| h ]$.  Finally, for a string $h \in \calO^t$ and character $o \in \calO$, we use $h \vee o$ to denote the string of $t+1$ characters obtained by appending $o$ to $h$.

\section{Technical Overview}\label{sec:tech-overview}


Our algorithm is composed of two main parts: We work with a certain representation of the distribution and show how to estimate it from conditional samples. Then we give an algorithm that takes this learned representation and can generate samples. In most learning problems, estimating the parameters is the challenging part and sampling is often trivial. But in our case, designing the sampling algorithm, which involves solving a sequence of convex optimization problems, is one of the key ingredients.



\subsection{Idealized Representation}\label{sec:ideal-intro}
In this section, we introduce the representation that will be central to our learning algorithm.  First we consider an idealized setting where we ignore issues like sampling noise and the fact that we cannot afford to work with exponentially large vectors directly.  We include this subsection for pedagogical reasons as many of the ideas are already in \cite{kakade2024learning}.  Nevertheless, it will be useful setup for explaining our new contributions in the following subsections.

We slightly abuse notation and for $h \in \calO^t$, let $\Pr_{\Hm}[\cdot | h]$ denote the vector whose entries are $\Pr_{\Hm}[f | h]$ as $f$ ranges over all elements of $\calO^{T - t}$.  Now Definition~\ref{def:low-rank}  tells us that as $h$ ranges over $\calO^t$, all of the vectors $\Pr_{\Hm}[\cdot | h]$ are contained in an $S$-dimensional subspace.  Thus, we only need to store $S$ of these vectors to ``span" the whole space.      

\paragraph{Barycentric Spanners} This leads us to the notion of a barycentric spanner, which is the key building block in the representation we will use:

\begin{definition}[Barycentric Spanner]
For a collection of vectors $z_1, \dots , z_n \in \R^d$, we say another set of vectors $v_1, \dots , v_s \in \R^d$ is a $(C,\gamma)$-spanner for them if for all  $j \in [n]$, there are coefficients $c_1, \dots , c_s$ with $|c_i| \leq C$ such that
\[
\norm{z_j - (c_1v_1 + \dots + c_sv_s)}_{1} \leq \gamma \,.
\]
\end{definition}

Crucially, for any set of vectors that are exactly contained in a $S$-dimensional subspace, there is always a subset of them that form a $(1, 0 )$ barycentric spanner for the full collection.  


\begin{fact}\label{fact:exist-spanner-intro}
Let $v_1, \dots , v_n$ be arbitrary vectors in $\R^d$,  Then there exists a subset $S \subseteq [n]$ with $|S| \leq d$ such that $\{ v_i\}_{i \in S}$ form a $(1, 0)$ spanner for the collection $(v_1, \dots , v_n)$.  
\end{fact}

Thus ideally, for each $0 < t < T$, we could store a subset $\calH^{(t)} \subseteq \calO^t$ with $|\calH^{(t)}| \leq S$ and such that  $\{ \Pr_{\Hm}[\cdot | h] \}_{h \in \calH^{(t)}}$ is a $(1,0)$-barycentric spanner for $\{ \Pr_{\Hm}[\cdot | h] \}_{h \in \calO^t}$.

\paragraph{Change-of-Basis} In addition to these spanning subsets, we will need to store some additional information about sampling probabilities and ``change-of-basis" between these spanning sets.  In particular we claim that if, in addition to the barycentric spanners, we also had the following information, then it would be enough to sample:
\begin{itemize}
\item The next character probabilities (under $\Hm$) for all of the elements of $\calH^{(t)}$
\item For each $0 \leq t < T$,  $o \in \calO$, $h \in \calH^{(t)}$, coefficients $\{\alpha_{h'}^{h \vee o}\}_{h' \in \calH^{(t+1)}}$ such that
\begin{equation}\label{eq:stored-coeffs}
\Pr_{\Hm}[\cdot | h \vee o] = \sum_{h' \in \calH^{(t+1)}} \alpha_{h'}^{h \vee o} \Pr_{\Hm}[\cdot | h']
\end{equation}
i.e. we know how to write $\Pr_{\Hm}[\cdot | h \vee o]$ as a linear combination of the vectors $\{\Pr_{\Hm}[\cdot | h']\}_{h' \in \calH^{(t+1)}}$. We refer to this as a ``change-of-basis". 
\end{itemize}

To see why these suffice, let's consider sampling a string $x = o_1o_2 \cdots $ one character at a time.  Assume we have sampled $x = o_1o_2 \cdots o_t$ so far.  Throughout the process we will maintain a set of coefficients  $\{ \alpha_h^{x}\}_{h \in \calH^{(t)}}$ such that
\begin{equation}\label{eq:spanner-representation-intro1}
\Pr_{\Hm}[\cdot | x ] \approx \sum_{h \in \calH^{(t)}} \alpha_h^{x} \Pr_{\Hm}[\cdot | h] \,.
\end{equation}
Since we know the next-character distributions for all of $h \in \calH^{(t)}$, the above allows us to exactly compute the next-character distribution for $x$.  After sampling a character, say $o$ from this distribution, we can then re-normalize to obtain
\begin{equation}\label{eq:conditional-probs-intro}
\Pr_{\Hm}[\cdot | x \vee o ] \approx \sum_{h \in \calH^{(t)}} \alpha_h^{x} \frac{\Pr_{\Hm}[o | h]}{\Pr_{\Hm}[o | x ]} \Pr_{\Hm}[\cdot | h \vee o] \,.
\end{equation}
We can then substitute \eqref{eq:stored-coeffs} into the above to rewrite the right hand side as a linear combination of the vectors $\{\Pr_{\Hm}[\cdot | h']\}_{h' \in \calH^{(t+1)}}$.  And now we can iterate this process again to sample the next character. The key point is that once we have a barycentric spanner, we do not need to store the coefficients $\alpha_h^{x}$ for each history $x$ but rather can track how they evolve.

\paragraph{Challenges} 
If we want to turn the above representation into a learning algorithm, there are some major obstacles. The vectors $\Pr_{\Hm}[\cdot | h]$ are exponentially large, and we only have approximate (but not exact) access to them.  When our estimates are noisy, the sampling approach above has issues as the error may grow multiplicatively after each iteration.  


\subsection{Learning the Representation}\label{sec:learning-intro}

\paragraph{Simulating pdf access to $\wh{\Hm}$ that is close to $\Hm$} Our first step is to show that, by using conditional queries, we can simulate exact p.d.f. access to a distribution $\wh{\Hm}$ that is $\eps$-close to $\Hm$ in TV distance for any inverse-polynomially small $\eps$ . 
\begin{lemma}[Informal, see Lemma~\ref{lem:pdf-estimation}]\label{lem:density-access-informal}
There is a distribution $\wh{\Hm}$ such that for all $h \in \calO^t$, the conditional distributions $\Pr_{\Hm}[\cdot | h]$ and $\Pr_{\wh{\Hm}}[\cdot | h]$ are $\eps$-close, and we can simulate query access to the exact p.d.f. of all conditional distributions of $\wh{\Hm}$.
\end{lemma}

Note that crucially this exact p.d.f. access is \textit{only for a distribution $\wh{\Hm}$ that is $\eps$-close to $\Hm$ and not for $\Hm$ itself} (as in the latter case \cite{kakade2024learning} already gives an algorithm for learning).  For most of the algorithm, we will work directly with $\wh{\Hm}$ and only use, implicitly in the analysis, that it is close to some low rank distribution $\Hm$.

\subsubsection{Dimensionality Reduction}\label{sec:overview-dim-reduction}  

Recall that our goal is to find a barycentric spanner for the set of vectors $\{ \Pr_{\Hm}[\cdot | h] \}_{h \in \calO^t}$.  These vectors are close to the vectors $\{ \Pr_{\wh{\Hm}}[\cdot | h] \}_{h \in \calO^t}$, so we will try to construct an approximate spanner for this latter collection.  The first problem is that these vectors have exponentially many entries.  In order to compute with them efficiently, we introduce a type of dimensionality reduction that allows us to subsample a subset of only polynomially many coordinates and only work with the restriction to those coordinates.  First, let $\calD$ be some distribution over strings in $\calO^{T - t}$.  Now we sample polynomially elements $f_1, \dots, f_m$ from $\calD$.  And for each $h \in \calO^t$, we associate $\Pr_{\wh{\Hm}}[\cdot | h]$ with the reduced vector
\begin{equation}\label{eq:dim-reduction-intro}
v_{h} = \left( \frac{\Pr_{\wh{\Hm}}[f_1| h]}{m\calD[f_1]}, \dots , \frac{\Pr_{\wh{\Hm}}[f_m| h]}{m\calD[f_m]} \right) 
\end{equation}
where $\calD[f_i]$ denotes the density of $\calD$ at $f_i$. Recall that by assumption, we have exact access to the densities $\Pr_{\wh{\Hm}}[f_i| h]$ in the numerator.  Thus as long as we also have exact density access to $\calD$, then the above entries can be computed exactly.  We first observe  that these vectors give unbiased estimates for linear functionals of the original distributions $\Pr_{\wh{\Hm}}[\cdot | h]$.

\begin{fact}\label{fact:dimensionality-reduction-intro}
For any distribution $\calD$ over $\calO^{T - t}$, in expectation over the random draws from $\calD$, for any subset $\calA \subseteq \calO^t$ and real coefficients $\{ c_{h} \}_{h \in \calA}$,
\[
\E\left[ \norm{\sum_{h \in \calA} c_h v_h}_1 \right] = \norm{\sum_{h \in \calA} c_h \Pr_{\wh{\Hm}}[\cdot | h]  }_1 \,.
\]
\end{fact}
In other words, at least in expectation, the subsampling of the entries preserves the $L^1$ norm of linear combinations of the vectors.  Of course, the variance could be prohibitively large depending on the choice of $\calD$.  However, with a careful choice of $\calD$, we can obtain concentration for the above quantities as well.  The goal of our dimensionality reduction procedure is captured by the following statements:

\begin{definition}
For a subset $\calA \subseteq \calO^t$, we say vectors  $\{ v_h\}_{h \in \calA}$ are $\gamma$-representative for the distributions $\{\Pr_{\wh{\Hm}}[\cdot | h ] \}_{h \in \calA}$ if  
\[
\left\lvert \norm{\sum_{h \in \calA} c_h v_h}_1 - \norm{\sum_{h \in \calA} c_h \Pr_{\wh{\Hm}}[\cdot | h]  }_1 \right\rvert \leq \gamma
\]
for all sets of coefficients with $|c_h| \leq 1$.  
\end{definition}

\begin{proposition}\label{prop:dim-reduce-intro}[Informal, see Lemma~\ref{lem:dimensionality-reduction} for more details]
For $\calA = \{h_1, \dots , h_k \}$ and 
\[
\calD = \frac{\Pr_{\wh{\Hm}}[\cdot | h_1] + \dots + \Pr_{\wh{\Hm}}[\cdot | h_k]}{k} \,.
\]
If we draw samples $f_1, \dots , f_m$ from $\calD$ for $m = \poly(k,1/\gamma)$ and set
\[
v_{h} = \left( \frac{\Pr_{\wh{\Hm}}[f_1| h]}{m\calD[f_1]}, \dots , \frac{\Pr_{\wh{\Hm}}[f_m| h]}{m\calD[f_m]} \right)
\]
then with high probability the vectors $\{v_h\}_{h \in \calA}$ are $\gamma$-representative for the distributions $\{ \Pr_{\wh{\Hm}}[\cdot | h] \}_{h \in \calA}$.
\end{proposition}

To see why the above holds, note that the entries of $v_{h_1}, \dots , v_{h_k}$ are all at most $k$.  Thus for any accuracy parameter $\gamma$, by taking $m > \poly(k, 1/\gamma)$, we can union bound over a net and get concentration in Fact~\ref{fact:dimensionality-reduction-intro}.  The main point now is that when we want to compute a barycentric spanner we can work with the representative vectors instead.  


\subsubsection{Computing a Spanner in the Reduced Space} 

Section~\ref{sec:overview-dim-reduction} deals with the fact that there are exponentially many possible futures.  There are also exponentially many histories, but because they are close to some low-dimensional subspace, we show that sampling a polynomial number of them suffices to obtain a representative subset.

Formally, we sample $k = \poly(S, 1/\gamma)$ histories $h_1, \dots , h_k$ from $\wh{\Hm}[:t]$.  While we cannot ensure that the collection 
\[
\{\Pr_{\wh{\Hm}}[\cdot | h_1], \dots , \Pr_{\wh{\Hm}}[\cdot | h_k] \}
\]
contains a barycentric spanner for all of $\{\Pr_{\wh{\Hm}}[\cdot | h] \}_{h \in \calO^t} $ (because there could be some $h$ that is sampled with very low probability), we can still use Fact~\ref{fact:exist-spanner-intro} to argue that this set contains a subset of size $S$ that is a barycentric spanner for \emph{most of} $\{\Pr_{\wh{\Hm}}[\cdot | h] \}_{h \in \calO^t} $ i.e. with high probability over $h$ drawn from $\wh{\Hm}$,  $\Pr_{\wh{\Hm}}[\cdot | h]$ is close to a bounded linear combination of vectors in this subset.

Now to compute the barycentric spanner, by Proposition~\ref{prop:dim-reduce-intro}, we can construct representative vectors $v_{h_1}, \dots , v_{h_k}$ and it suffices to work with these, which have polynomial size.  Then we can run an algorithm for computing approximate barycentric spanners (see Lemma~\ref{lem:compute-spanner-full}) on this collection.  There are additional technical details due to the error i.e. the vectors are close to, but not exactly contained in an $S$-dimensional subspace.  For the algorithm to work, it is crucial that the error is smaller than the dimensionality of the vectors so we need to ensure that $\eps \ll \poly(\gamma)$ (see Section~\ref{sec:learning-spanner} for more details). Omitting these technical details, we have now sketched the proof of the following:

\begin{lemma}[Informal, see Lemma~\ref{lem:compute-spanner-full}]\label{lem:learn-spanners-intro}
With $\poly(S,O,T,1/\gamma)$ conditional queries and runtime, we can compute sets  $\calH^{(1)}, \dots , \calH^{(T -1)}$ such that for each $1 \leq t \leq T - 1$, $|\calH^{(t)}| \leq S$ and the collection $\{\Pr_{\wh{\Hm}}[\cdot | h] \}_{h \in \calH^{(t)}}$ is a $(O(1), \gamma)$-spanner for a $1 - \gamma$-fraction of $\{\Pr_{\wh{\Hm}}[\cdot | h] \}_{h \in \calO^t}$ (where the fraction is measured with respect to $h \sim \wh{\Hm}$).    
\end{lemma} 
For simplicity, in this overview, we even assume that $\{\Pr_{\wh{\Hm}}[\cdot | h] \}_{h \in \calH^{(t)}}$ is actually an approximate spanner for all of $\{\Pr_{\wh{\Hm}}[\cdot | h] \}_{h \in \calO^t}$ \---- in the full proof, roughly the ``exceptional" possibilities for $h$ can be absorbed into the failure probability.

\subsubsection{A Reduced Representation}



Let $\calB^{(t)} \subseteq \calO^t$ consist of all strings in $\calH^{(t)}$ and all strings obtained by taking a string in $\calH^{(t-1)}$ and appending a character in $\calO$.  In other words $\calB^{(t)} = \calH^{(t)} \cup \{ \calH^{(t-1)} \vee o\}_{o \in \calO}$.  Now the set of information that we store is as follows:

\begin{definition}[Learned representation (informal)]
Our learned representation consists of the following information:
\begin{itemize}
\item Sets $\calH^{(t)}$ for $t = 1,2, \dots , T - 1$
\item The next character probabilities (under $\wh{\Hm}$) for all of the elements of $\calH^{(t)}$
\item Vectors $\{ v_h \}_{h \in \calB^{(t)}}$ that are representative for the distributions $\{ \Pr_{\wh{\Hm}}[ \cdot | h] \}_{h \in \calB^{(t)}}$ (can be constructed using Proposition~\ref{prop:dim-reduce-intro}).
\end{itemize}
\end{definition}
This is our first main departure from the idealized representation: rather than storing explicit linear combinations as in the idealized representation (see \eqref{eq:stored-coeffs}), we simply store the representative vectors and defer the computation of the ``change-of-basis" to the sampling algorithm. For technical reasons, the full algorithm requires storing some additional information but we will omit these in this overview.  See Section~\ref{sec:full-learning} for details.

\subsection{Sampling Step}\label{sec:sample-intro}
As in the idealized sampling algorithm in Section~\ref{sec:ideal-intro}, we sample a string $x$ one character at a time.  At each step $t < T$, we maintain a linear combination $\{ \alpha_h^{x} \}_{h \in \calH^{(t)}}$ of the strings $h \in \calH^{(t)}$ as in \eqref{eq:spanner-representation-intro1}, except with conditional distributions with respect to $\wh{\Hm}$.  We then sample a character and re-normalize as in \eqref{eq:conditional-probs-intro}. The main difficulty now is the change-of-basis step.


 \paragraph{Naive Attempt: Linear Change-of-Basis} Naively, for each $h \in \calH^{(t)}$ and $o \in \calO$, we could simply pre-compute coefficients 
  $\{\beta_{h'}^{h \vee o} \}_{h' \in \calH^{(t+1)}}$ such that 
 \[
 v_{h \vee o} \approx \sum_{h' \in \calH^{(t+1)}} \beta_{h'}^{h \vee o} v_{h'}
 \]
 and thus (since the vectors $v_h$ are representative for the distribution $\Pr_{\wh{\Hm}}[\cdot | h]$), 
\[
\Pr_{\wh{\Hm}}[\cdot | h \vee o]  \approx \sum_{h' \in \calH^{(t+1)}} \beta_{h'}^{h \vee o}  \Pr_{\wh{\Hm}}[\cdot | h'] \,.
\]
We can then substitute this into the re-normalized relation (see \eqref{eq:conditional-probs-intro}) and get
\begin{equation}\label{eq:linear-transition-intro}
\Pr_{\wh{\Hm}}[\cdot | x \vee o ] \approx  \sum_{h' \in \calH^{(t+1)}} \theta_{h'}^{x \vee o} \Pr_{\wh{\Hm}}[\cdot | h'] \mbox{, where } \theta_{h'}^{x \vee o} = \sum_{h \in \calH^{(t)}}\alpha_{h}^{x} \beta_{h'}^{h \vee o} \cdot \frac{\Pr_{\wh{\Hm}}[ o | h]}{\Pr_{\wh{\Hm}}[o | x]}
\end{equation}
This gives us an expression for $\Pr_{\wh{\Hm}}[\cdot | x \vee o ] $ as a linear combination of the vectors $\Pr_{\wh{\Hm}}[\cdot | h']$ for $h' \in \calH^{(t+1)}$.  However, the issue with this approach is that the coefficients and approximation error may grow exponentially with $t$.

\paragraph{Need for Projection}

The fact that the coefficients may grow motivates the need for a ``projection" step, where we reduce the magnitudes of the coefficients.  One natural attempt is to take the coefficients $\{ \theta_{h'}^{x \vee o}\}_{h' \in \calH^{(t+1)}}$ in \eqref{eq:linear-transition-intro} and ``project" them by finding an equivalent set of coefficients of bounded magnitude, say $\{ \alpha_{h'}^{x \vee o}\}_{h' \in \calH^{(t+1)}}$, such that 
\begin{equation}\label{eq:projection-attempt-1}
\sum_{h' \in \calH^{(t+1)}} \theta_{h'}^{x \vee o} \Pr_{\wh{\Hm}}[\cdot | h'] \approx \sum_{h' \in \calH^{(t+1)}}   \alpha_{h'}^{x \vee o} \Pr_{\wh{\Hm}}[\cdot | h'] \,.
\end{equation}
In fact, there exists a set of coefficients of bounded magnitude such that 
\begin{equation}
\Pr_{\wh{\Hm}}[\cdot | x \vee o ]  \approx \sum_{h'  \in \calH^{(t+1)}}    \alpha_{h'}^{x \vee o}  \Pr_{\wh{\Hm}}[\cdot | h'] 
\end{equation}
with high probability over $x$, by Lemma~\ref{lem:learn-spanners-intro}.  However, there is a subtle issue with this approach. We only have access to the vectors $\{v_{h}\}_{h \in \calB^{(t +1)}}$ so we can only compute (succinct representations) of the expressions in \eqref{eq:projection-attempt-1} but not of $\Pr_{\wh{\Hm}}[\cdot | x \vee o]$.  

The issue is captured by the following abstraction.  Let $w$ be the vector that is equal to the left hand side of \eqref{eq:projection-attempt-1}.  Let $\calK$ be the convex set consisting of all possible vectors obtainable on the right hand side of \eqref{eq:projection-attempt-1} for coefficients $\{ \alpha_{h'}^{x \vee o}\}_{h' \in \calH^{(t+1)}}$ bounded by some constant.  Let $z = \Pr_{\wh{\Hm}}[\cdot | x \vee o ]$.  We know $z \in \calK$ (up to some small error), but we do not know $z$.  Now we want to map $w$ to a point in $z' \in \calK$.  Ideally, we want to guarantee $\norm{z' - z}_1 \leq \norm{w - z}_1 $, up to some small additive error.  The issue is that we cannot guarantee better than
\[
\norm{z' - z}_1 \leq 2\norm{w - z}_1 
\]
due to the fact that $z$ is unknown and having to use triangle inequality. In other words, the problem is that when projecting onto a convex set in $L^1$ norm, we may still double the $L^1$ distance to some of the elements of the convex set, causing the error to grow multiplicatively in each iteration.

\paragraph{Projection in KL}

The above issue motivates projection in a different distance measure between distributions.  Specifically, KL divergence has the following appealing property:
\begin{fact}\label{fact:kl-basic-intro}
Let $\calT \subset \R^N$ be the $N$-dimensional simplex i.e. the convex hull of the standard basis vectors.  We view points in $\calT$ as distributions over $N$ elements.  Let $\calK \subseteq \calT$ be a convex set and let $w \in \calT$ be a point.  Let $z^* = \argmin_{z \in \calK} \KL(z \| v)$.  Then for all $z \in \calK$,
\[
\KL(z \| z^*) \leq \KL(z \| w) \,. 
\]    
\end{fact}
This implies that \emph{``projecting" onto a convex set actually decreases the KL-divergence from all elements within the convex set}.  This is the key to overcoming the error doubling issue above. In particular, in the same abstraction as above, we would just set $z' = z^*$.

More concretely, since the vectors $\{v_h\}_{h \in \calB^{(t+1)}}$ are succinct representations of $\Pr_{\wh{\Hm}}[\cdot | h]$, we can solve the projection in KL, which is a convex optimization problem, over these representative vectors. While Proposition~\ref{prop:dim-reduce-intro} is stated for TV distance, we can extend it to also preserve the KL divergence between linear combinations of these distributions as well.  Technically, this requires a ``truncated" notion of KL to deal with possibly negative entries but we will gloss over this for now (see Definition~\ref{def:truncated-kl-function} and Lemma~\ref{lem:dimensionality-reduction} for more details). 

To describe our actual algorithm, once we have the coefficients as in \eqref{eq:conditional-probs-intro}, we solve the following convex optimization problem.  We solve for coefficients $\{ \alpha_{h'}^{x \vee o} \}_{h' \in \calH^{(t+1)}}$ that are bounded in magnitude and minimize
\begin{equation}\label{eq:opt-intro}
\KL \left( \sum_{h' \in \calH^{(t+1)}} \alpha_{h'}^{x \vee o}v_{h'} \Bigg\| \sum_{h \in \calH^{(t)}} \alpha_h^{x} \frac{\Pr_{\wh{\Hm}}[ o | h]}{\Pr_{\wh{\Hm}}[o | x]} \Pr_{\wh{\Hm}}[\cdot | h \vee o]  \right) \,.
\end{equation}
While we omit the details in this overview (the full analysis is in Section~\ref{sec:sampling-alg}), the main point of this projection in KL step is that we can appeal to Fact~\ref{fact:kl-basic-intro} with $z =  \Pr_{\wh{\Hm}}[\cdot | x \vee o ]$ and $z^* = \sum_{h' \in \calH^{(t+1)}} \alpha_{h'}^{x \vee o}v_{h'}$ for the optima obtained in \eqref{eq:opt-intro} to get
\[
\KL \left( \Pr_{\wh{\Hm}}[\cdot | x \vee o ] \Bigg\|   \sum_{h' \in \calH^{(t+1)}} \alpha_{h'}^{x \vee o}v_{h'}  \right) \leq \KL \left(  \Pr_{\wh{\Hm}}[\cdot | x \vee o ] \Bigg\| \sum_{h \in \calH^{(t)}} \alpha_h^{x} \frac{\Pr_{\wh{\Hm}}[ o | h]}{\Pr_{\wh{\Hm}}[o | x]} \Pr_{\wh{\Hm}}[\cdot | h \vee o]   \right) + \gamma
\]
where the additive error $\gamma$ comes only from how well KL-divergences are preserved in the succinct representation $\Pr_{\wh{\Hm}}[\cdot | h] \rightarrow v_h$.  In particular, this completes the ``change-of-basis" from $\calH^{(t)}$ to $\calH^{(t+1)}$ and \emph{the error only grows additively}.  Now to sample the entire string, we simply sample one character at a time and iterate the above. 

\subsection{Organization}

In Section~\ref{sec:basics}, we present a few basic observations for simulating density access to a distribution $\wh{\Hm}$ that is close to $\Hm$ and working with (truncated) KL divergence.  In Section~\ref{sec:geometry}, we introduce our machinery for working with barycentric spanners, including the dimensionality reduction procedure.  In Section~\ref{sec:learning-main}, we present our full learning algorithm.  In Section~\ref{sec:sampling-alg}, we present our sampling algorithm for sampling from our learned representation.

\section{Basic Results}\label{sec:basics}

In this section, we introduce notation and collect several basic observations that will be used throughout the later sections.

\subsection{Algorithmic Primitives}

We start by developing a few basic primitives that conditional queries allow us to implement.  First, we have a subroutine for computing the probabilities of the next character given a specified history $h$.

\begin{claim}[Next Character Probabilities]\label{claim:next-char}
For $t \leq T$ and a string $h \in \calO^t$ and parameters $\eps, \delta > 0$, we can make $\poly(1/\eps, O, \log(1/\delta))$ conditional queries, and with $1 - \delta$ probability, we output a distribution $\calP_h$ over $\calO$ that is within TV distance $\eps$ of the true distribution of the next character $\{ \Pr_{\Hm}[ o |h] \}_{o \in \calO}$.
\end{claim}
\begin{proof}
We can repeatedly make conditional queries with input $h$ and look at the next character.  We make $\poly(1/\eps, O,  \log(1/\delta))$ such queries and output the empirical distribution of the next character.  By a Chernoff bound, with $1 - \delta$ probability, the resulting distribution is within TV distance $\eps$ of the true distribution $\{ \Pr_{\Hm}[ o |h] \}_{o \in \calO}$.
\end{proof}

In light of the above, we make the following definition.
\begin{definition}[Conditional Closeness]
For two distributions $\Hm, \wh{\Hm}$ on $\calO^T$, we say they are $\eps$-conditionally close if  for any string $h \in \calO^t$ for $t < T$, the conditional distributions of  the next character  $\{ \Pr_{\Hm}[ o |h] \}_{o \in \calO}$ and  $\{ \Pr_{\wh{\Hm}}[ o |h] \}_{o \in \calO}$ are $\eps$-close in TV distance.
\end{definition}

The following statement shows that conditional closeness implies closeness in TV distance for the entire distribution (and further on any conditional distribution) up to a factor of $T$.
\begin{claim}\label{claim:nextchar-to-TV}
Let $\Hm, \wh{\Hm}$ be distributions on $\calO^T$ that are $\eps/T$ conditionally close.  Then for any history $h \in \calO^t$ for $t \leq T$, we must have 
\[
d_{\TV}(\Pr_{\Hm}[ \cdot | h], \Pr_{\wh{\Hm}}[ \cdot | h]) \leq \eps 
\]
where $\Pr_{\Hm}[ \cdot | h], \Pr_{\wh{\Hm}}[ \cdot | h]$ represent the full distribution over the future $f \in \calO^{T - t}$.
\end{claim}
\begin{proof}
We prove by reverse induction on $t$ that 
\[
d_{\TV}(\Pr_{\Hm}[ \cdot | h], \Pr_{\wh{\Hm}}[ \cdot | h]) \leq \frac{\eps (T - t)}{T}
\]
whenever $h$ has length $t$.  When $t = T-1$, the above is immediate from the assumption that $\Hm, \wh{\Hm}$ are $\eps/T$ conditionally close.  Now assume we have proven the claim for $t$.  Consider $h$ that has length $t-1$.  
\[
d_{\TV}(\Pr_{\Hm}[ \cdot | h], \Pr_{\wh{\Hm}}[ \cdot | h]) \leq \frac{\eps}{T} + \sum_{o \in \calO} \Pr_{\Hm}[ o | h] d_{\TV}(\Pr_{\Hm}[ \cdot | h \vee o], \Pr_{\wh{\Hm}}[ \cdot | h \vee o]) \leq \frac{\eps (T - t+ 1)}{T} 
\]
where in the above, we first used conditional closeness and then the inductive hypothesis.  This completes the proof.
\end{proof}

For technical reasons later on, it will be convenient to have the following definition which necessitates that, regardless of the prefix, all possibilities for the next character occur with some non-negligible probability.
\begin{definition}[Positivity]\label{def:positivity}
Let $0 <\eta < 1$ be some parameter.  We say a distribution $\Hm$ on $\calO^T$ is $\eta$-positive if for any string $h \in \calO^t$ for $t < T$, the probabilities $\Pr_{\Hm}[o|h]$ are at least $\eta$ for any choice of $o \in \calO$.
\end{definition}

Now we show how to implement sample and exact pdf access to a distribution $\wh{\Hm}$ that is conditionally close to the distribution of $\Hm$.  Note that this is slightly stronger than just implementing approximate pdf access to $\Hm$ since we need some ``consistency" between the responses to pdf queries. 

First, we can approximate the pdf of $\Hm$ at a given string $h \in \calO^t$ by iteratively computing the next character probabilities using Claim~\ref{claim:next-char} and then multiplying them.  The key to ensure consistency is to only estimate the ``next character distribution" once for each possible prefix $h$.  Whenever this prefix shows up again in a different computation, we use the same ``next character distribution" that has already been computed.  

\begin{definition}[Sample and PDF Access]
For a distribution $\wh{\Hm}$ over $\calO^T$, we say that we have sample and pdf access to $\wh{\Hm}$ if we can perform the following operations:
\begin{itemize}
\item Given a string $h \in \calO^t$ for $t < T$, draw a sample $f \in \calO^{T - t}$ from $\Pr_{\wh{\Hm}}[f| h]$
\item Given a string $h \in \calO^t$ for $t \leq T$, output $\Pr_{\wh{\Hm}}[h]$ (where this is the probability that the first $t$ characters of the string match $h$)
\end{itemize}
\end{definition}

\begin{lemma}[PDF Estimation]\label{lem:pdf-estimation}
Let $N \in \N$ and $\eps , \delta > 0$ be parameters.  Assume we are given conditional query access to a distribution $\Hm$.  With probability at least $1 - \delta$, we can respond to $N$ sample and pdf queries for some distribution $\wh{\Hm}$ that is $\eps/(10O)^2$-positive and $\eps$ conditionally close to $\Hm$.  This uses a total of $\poly(1/\eps, T, O , N, \log(1/\delta))$ conditional queries to $\Hm$.
\end{lemma}
\begin{remark}
While we don't specify the whole distribution $\wh{\Hm}$, the point is that the responses made by the algorithm to the sample and pdf queries are consistent with some distribution $\wh{\Hm}$ that is close to $\Hm$. 
\end{remark}
\begin{proof}
We define a rooted tree with $T$ layers where each intermediate node has exactly $O$ children.  Now the nodes of the tree are labeled with strings in $\calO^t$ for $t =0,1, \dots , T$.  The root is labeled with the empty string and the labels of the children of each node are obtained by appending each of the $O$ possible characters \--- thus the labels of the nodes at level $t$ are exactly the strings in $\calO^t$.

We first describe an algorithm for answering sample and pdf queries that requires exponential time but then we will give a polynomial time algorithm for answering polynomially many queries that is statistically equivalent to it.

\paragraph{Naive Algorithm:} for every non-leaf node in the tree, say with label $h$, apply Claim~\ref{claim:next-char} (using new, independent samples) to compute the next-character probabilities $p_{h}[o]$ for $o \in \calO$.  We can ensure that with probability $1 - \delta/O^T$, $|\Pr_{\Hm}[o|h] - p_{h}[o]| \leq \eps/(10O)$ for all $o \in \calO$.  By perturbing the $p_h[o]$ if necessary, we can also ensure $p_{h}[o] \geq \eps/(10O)^2$ and $\sum_{o \in \calO} p_{h}[o] = 1$ for all $o \in \calO$.  Now assign the edge from $h$ to $h \vee o$ a weight of $p_{h}[o]$.  Let $\wh{\Hm}$ be the distribution induced by this weighted tree where the probability of sampling a string $s \in \calO^T$ is equal to the product of the weights along the root-to-leaf path to $s$.  Now answer all sample and pdf queries with respect to $\wh{\Hm}$.  It is clear that $\wh{\Hm}$ is a valid distribution, and union bounding over all nodes, with $1 - \delta$ probability, $\wh{\Hm}$ is $\eps/(10O)^2$-positive and $\eps$ conditionally close to $\Hm$.

\paragraph{Efficient Implementation:} We say we visit a node labeled $h$ to mean that we apply Claim~\ref{claim:next-char} to compute the next-character probabilities and assign edge weights as described above. We will lazily maintain the tree instead, only visiting nodes as necessary.  After we have visited a node, we never revisit it, so its edge weights are fixed once it has been visited once.  Initially, we have only visited the root.  Each time we receive a pdf query for a string $h$, we follow the root-to-leaf path to $h$ and visit all nodes along this path.  Now the pdf we output is obtained by simply multiplying the edge weights along the path from the root to $h$.  

Next, for a conditional sample query at $h$, we begin from the node labeled $h$.  If we have visited it already, then we sample one of its children with probabilities proportional to the edge weights.  We then move to this child and repeat.  If we have not visited a node yet, we first visit it and then sample a child and repeat until we reach a leaf.  We output this leaf as the conditional sample.

In this implementation, answering each query only requires $\poly(1/\eps, T, O, \log(1/\delta))$ time and conditional queries to $\Hm$.  Furthermore, this algorithm is statistically equivalent to the naive algorithm described earlier so we are done.

\end{proof}

\subsection{Truncated KL Divergence}

Later on in the analysis, we will need a modified version of KL Divergence that is truncated so that it is defined everywhere and doesn't blow up at $0$.  

\begin{definition}\label{def:truncated-kl-function}
For a parameter $c > 0$ and real valued inputs $x,y$, we define
\[
\ell_{\geq  c}(x,y) =  x \log \max(x,c) - x \log \max(y,c) \,.
\]
\end{definition}

\begin{definition}
For vectors $u, v, w \in \R^N$ where $w$  has positive entries, we define 
\[
\KL_{\geq w }(u \| v) = \sum_{i = 1}^N \ell_{\geq w[i]}(u[i], v[i])\,.
\]
When $w$ has all entries equal to $c$, we will slightly simplify notation and write 
\[
\KL_{\geq c}(u \| v) = \sum_{i = 1}^N \ell_{\geq c}(u[i], v[i]) \,.
\]
\end{definition}

Note that the above is well-defined since we are only evaluating the logarithm when both $u[i]$ and $v[i]$ are positive.  When $u,v$ are distributions and $c = 0$, the above corresponds to the usual definition of KL-divergence.  When $c > 0$, this corresponds to a soft truncation of the distributions to elements with nontrivial mass.

One of the key properties of KL divergence is that it allows for projection in the following sense: when projecting a point onto a convex set, the projection is closer to \emph{all} points in the set.
\begin{fact}\label{fact:KL-basic}
Let $\calT \subset \R^N$ be the $N$-dimensional simplex i.e. the convex hull of the standard basis vectors.  Let $\calK \subseteq \calT$ be a convex set and let $v \in \calT$ be a point.  Let $x^* = \argmin_{x \in \calK} \KL(x \| v)$.  Then for all $x \in \calK$,
\[
\KL(x \| x^*) \leq \KL(x \| v) \,. 
\]
\end{fact}
\begin{proof}
Let $x^* = (x_1^*, \dots , x_N^*)$ and $v = (v_1, \dots , v_N)$.  Then the optimality of $x^*$ implies that for any $x = (x_1, \dots , x_N) \in \calK$, 
\[
\sum_{j =1}^N (x_j - x_j^*) (\log x_j^* - \log v_j + 1) \geq 0 \,.
\]
Since $x, x^*$ are on the simplex and $\sum_{j =1}^N x_j^* (\log x_j^* - \log v_j) \geq 0$ by non-negativity of KL divergence, we conclude
\[
\sum_{j =1}^N x_j \log v_j \leq \sum_{j =1}^N x_j \log x_j^*
\]
and this immediately gives the desired inequality.
\end{proof}

In our analysis later on, we will need a slightly more general version of Fact~\ref{fact:KL-basic}, stated below, that deals with truncated KL divergences and points that are not exactly on the simplex.

\begin{corollary}\label{coro:KL-basic}
Let $w ,v \in \R^N$ be vectors.  Assume $w$ has positive coordinates and assume $\calK \subset \R^N$ is a convex set such that all elements $x \in \calK$ have sum of coordinates equal to $1$ and $x \geq w$ entrywise.  Let $v \in \R^N$ be a point.  Let $x^* = \argmin_{x \in \calK} \KL_{\geq w}(x \| v)$.  Then for any $x \in \calK$,
\[
\KL_{\geq w}(x \| x^*) \leq \KL_{\geq w}(x \| v)  +  \log (\norm{v}_1 + \norm{w}_1) \,.
\]
\end{corollary}
\begin{proof}
Note that by the assumptions in the statement, for any $x \in \calK$,
\[
\KL_{\geq w}(x \| v) = \KL(x \| \max(w,v)) = \KL\left(x \| \max(w,v)/ C \right) - \log C
\]
where the max on the RHS is taken entrywise and $C$ is any positive constant. 
 Now letting $C$ be equal to the sum of the coordinates of $\max(w,v)$, note that $C \leq \norm{v}_1 + \norm{w}_1$. Then we can simply apply Fact~\ref{fact:KL-basic} to the vector $\max(w,v)/ C $ to get the desired inequality.
\end{proof}

\noindent We also have the following observation that will allow us to relate truncated KL divergence to TV distance.

\begin{claim}\label{claim:trunc-KL-lipchitz}
For a vector $w \in \R^N$ with positive entries and minimum entry $w_{\min}$ and maximum entry $w_{\max}$, the divergence $\KL_{\geq w}(u \| v)$ is Lipchitz with respect to $u$ with Lipchitz constant 
\[
1 +  \log (1 + 1/w_{\min}) + \log(1 + w_{\max}) + \log (1 + \norm{u}_{\infty})  + \log (1 + \norm{v}_{\infty}) \,.
\]
\end{claim}
\begin{proof}
This follows immediately from the definition.   
\end{proof}


\section{Geometric Results}\label{sec:geometry}

Our learning algorithm will also rely on a few geometric constructions which we introduce below.

\subsection{Spanners}\label{sec:spanners}

First, we recall the standard notion of a barycentric spanner.  

\begin{definition}[Barycentric Spanner]
For a collection of vectors $z_1, \dots , z_n \in \R^d$, we say another set of vectors $v_1, \dots , v_s \in \R^d$ is a $(C,\gamma)$-spanner for them if for all  $j \in [n]$, there are coefficients $c_1, \dots , c_s$ with $|c_i| \leq C$ such that
\[
\norm{z_j - (c_1v_1 + \dots + c_sv_s)}_{1} \leq \gamma \,.
\]
\end{definition}

Note that the error is defined in $L^1$ because later on, we will have $v_1, \dots , v_s$ representing the density functions of distributions and the error will correspond to TV-closeness.

Next we introduce a distributional notion of a spanner, where we only require that the spanner covers most of the mass of some underlying distribution.

\begin{definition}[Distribution Spanner]
Let $\calD$ be a distribution over $\R^d$.  Let $C > 0$, $0 \leq \eps, \gamma \leq 1$ be parameters.  We say a set of points $\{v_1, \dots , v_s \} \in \R^d$ is a $(1 - \eps , C, \gamma)$-spanner for $\calD$ if for a random point $z$ drawn from $\calD$, with $1 - \eps$ probability, there are coefficients $|c_1|, \dots , |c_s| \leq C$ with 
\[
\norm{z - (c_1 v_1 + \dots + c_s v_s)}_{1} \leq \gamma \,.
\]
\end{definition}

It is folkore that any subset of points has an exact barycentric spanner.  However, in our setting, it will be important to compute a spanner for a high-dimensional distribution given only polynomially many samples.

\begin{fact}\label{fact:spanner}[Folklore]
Let $v_1, \dots , v_n$ be arbitrary vectors in $\R^d$,  Then there exists a subset $S \subseteq [n]$ with $|S| \leq d$ such that $\{ v_i\}_{i \in S}$ form a $(1, 0)$ spanner for the collection $(v_1, \dots , v_n)$.  
\end{fact}

\begin{lemma}\label{lem:distribution-spanner}
Let $\calD$ be a distribution over $\R^d$.  Let $0 < \eps ,\delta < 1$ be parameters.  Given a set $\calA$ of $(d/\eps)^{10} \cdot \log(1/\delta)$ independent samples from $\calD$, with probability $1 - \delta$, there exists a subset of $d$ elements of $\calA$, $\{ v_1, \dots  ,v_d \}$, which form a $(1 - \eps , 1, 0 )$-spanner for $\calD$. 
\end{lemma}
\begin{proof}
For any $d$ vectors in $\R^d$, say $v_1, \dots, v_d$, let $\calP_{v_1, \dots , v_d}$ be the convex body whose vertices are $\pm v_1 \pm \dots \pm v_d$.  Let $\mu(\calP_{v_1, \dots , v_d})$ be the pdf of $\calD$ inside this convex body and let $r(\calP_{v_1, \dots , v_d})$ be the fraction of points of $\calA$ that are within this body.  For fixed, $v_1, \dots , v_d$, we have
\[
|\mu(\calP_{v_1, \dots , v_d}) - r(\calP_{v_1, \dots , v_d})| \leq 0.5\eps 
\]
with $1 - \delta^{(d/\eps)^2}$ probability.  Now we can union bound over all subsets $\{ v_1, \dots , v_d \} \subseteq \calA$ and deduce that with probability $1 - \delta$, for all such subsets, 
\[
|\mu(\calP_{v_1, \dots , v_d}) - r(\calP_{v_1, \dots , v_d})| \leq \eps \,.
\]
However, by Fact~\ref{fact:spanner}, there is a subset of $d$ elements of $\calA$ that is a $(1,0)$-spanner for $\calA$ and thus the above implies that taking $\{ v_1, \dots , v_d \}$ to be this subset gives us a $(1 - \eps , 1, 0)$-spanner for $\calD$. 
\end{proof}

Now we recall the following standard algorithm for computing a spanner (see e.g. \cite{awerbuch2008online}).  We include a proof for completeness.

\begin{claim}\label{claim:compute-spanner-basic}
Let $v_1, \dots , v_n \in \R^d$ be vectors and let $\mM \in \R^{d \times n}$ be the matrix with columns given by $v_1, \dots v_n$ and let its singular values be $\sigma_1 \geq \dots \geq \sigma_d$.  Then there is an algorithm that runs in time $\poly(n,d) \cdot \log(\sigma_1/\sigma_d)$ that computes a subset $\calB \subseteq [n]$ with $|\calB| \leq d$ such that $\{v_i \}_{i \in \calB}$ is a $(2, 0)$ spanner for the collection $\{v_1, \dots , v_n \}$
\end{claim}
\begin{proof}
First, we  iteratively construct an initial subset $\calA \subseteq [n]$ as follows.  Initially start with $\calA_0 = \emptyset$ and in each step $j = 1,2, \dots , d$, find a vector $v_{a_j}$ such that 
\[
\norm{\Pi_{\textsf{span}\left(\{v_i\}_{i \in \calA_{j-1}})\right)^{\perp}}(v_{a_j})} \geq \frac{\sigma_d}{\sqrt{n}} \,.
\]
First, such a vector always exists for $j =1,2,\dots , d$ because we know that
\[
\sum_{a \in [n]} \norm{\Pi_{\textsf{span}\left(\{v_i\}_{i \in \calA_{j-1}})\right)^{\perp}}(v_a)}^2 \geq \min_{\mU \in \R^{d \times (j-1)}, \mU^\top \mU = \mI }\norm{\mM - \mU \mU^\top \mM}_F^2 \geq \sigma_j^2 \geq \sigma_d^2
\]
where we used the extremal characterization of the $j$th singular value of $\mM$ in terms of the best rank-$j$ approximation to $\mM$.

After this procedure, we have a subset $\calA = \{v_{a_1}, \dots , v_{a_d} \}$.  We will use $\textsf{Vol}(\calA)$ to denote the volume of the simplex formed by the vectors in $\calA$.  By construction, we have
\[
\textsf{Vol}(\calA) \geq \frac{\sigma_d^d}{(dn)^d} \,.
\]
Next, assume for the sake of contradiction that $\calA$ is not a $(2,0)$-spanner.  Then there is some vector $v_i$ that cannot be written as a linear combination of $v_{a_1}, \dots , v_{a_d}$ with coefficients of magnitude at most $2$.  There is a unique way to write 
\[
v_i = c_1v_{a_1} + \dots + c_d v_{a_d}
\]
and there must be some coefficient say $c_j$ with $|c_j| \geq 2$.  Then we can replace $v_{a_j}$ with $v_i$ in $\calA$ and this at least doubles $\textsf{Vol}(\calA)$.  The maximal possible value of   $\textsf{Vol}(\calA)$ for any subset $\calA$ with $|\calA| = d$ is at most $\sigma_1^d$ and thus we need to iterate this process at most $\poly(n,d)  \cdot \log(\sigma_1/\sigma_d)$ times before we get a set $\calA$ that is a $(2,0)$-spanner.  This completes the proof.
\end{proof}

Claim~\ref{claim:compute-spanner-basic} has the restriction that $\sigma_d$ is nontrivial.  We will now give a robust version of Claim~\ref{claim:compute-spanner-basic} for computing an approximate spanner even when the vectors may be arbitrarily close to degenerate.

\begin{algorithm}[ht!]
\caption{Robust Spanner}
\begin{algorithmic}[1]
\State \textbf{Input:} Vectors $v_1, \dots , v_n \in \R^d$
\State \textbf{Input:} Parameters $s, \gamma$
\State Form matrix $\mM \in \R^{d \times n}$ with columns $v_1, \dots , v_n$
\State Compute SVD of $\mM = \mU \msig \mV^\top$ 
\State Sort the singular values $\sigma_1 \geq \sigma_2 \geq \dots \geq \sigma_d$ 
\State Let $\sigma_t$ be the largest index with $\sigma_t > \gamma \sqrt{n}$
\State Let $\mY$ be the matrix formed by the columns of $\mU$ corresponding to the top-$t$ singular values
\State Define $u_i = \mY^\top v_i$ for all $i \in [n]$
\State Apply Claim~\ref{claim:compute-spanner-basic} to compute $(2,0)$-spanner $\{u_{a_1}, \dots , u_{a_t} \}$ for $\{u_1, \dots , u_n \}$ \label{line:spanner-step}
\State \textbf{Output:} $\{a_1, \dots , a_t\}$
\end{algorithmic}
\label{alg:robust-spanner}
\end{algorithm}

\begin{lemma}\label{lem:compute-spanner-full}
Let $v_1, \dots , v_n \in \R^d$ be vectors with $\norm{v_{i}} \leq \poly(n,d)$ for all $i$.  Let $s, \gamma$ be parameters we are given.  Assume that there exists a subspace $\mV$ of dimension $s$ such that $\norm{\Pi_{\mV^\perp} v_i} \leq \gamma$ for all $i$.    Then if we run Algorithm~\ref{alg:robust-spanner} on $v_1, \dots, v_n$, it runs in time $\poly(n,d) \cdot \log(1/\gamma)$  and its output satisfies 
\begin{itemize}
    \item $t \leq s$
    \item $\{v_{a_1}, \dots , v_{a_t} \}$ forms a $(2, 3\gamma s\sqrt{nd})$ spanner for the collection $\{v_1, \dots , v_n \}$
\end{itemize}
\end{lemma}
\begin{proof}
Recall that $t$ is the largest index such that $\sigma_t > \gamma \sqrt{n}$.  Note that $t \leq s$ since by assumption, there is rank-$s$ subspace say $\mW \in \R^{d \times s}$ such that 
\[
\norm{\mM - \mW \mW^\top \mM}_F \leq \sqrt{n}\gamma 
\]
and thus $\sigma_{s+1} \leq  \sqrt{n}\gamma$.

Next, note that $\mY^\top \mM$ is a $t \times n$ matrix with all $t$ singular values being at least $\gamma \sqrt{n}$ (and the largest singular value is at most $\poly(n,d)$).  Thus, when we apply Claim~\ref{claim:compute-spanner-basic} in Line~\ref{line:spanner-step}, it runs in $\poly(n,d) \cdot \log(1/\gamma)$ time and the subset $\{u_{a_1}, \dots , u_{a_t} \}$ is indeed a $(2,0)$-spanner for the $u_i$.  Finally, we show that $\{v_{a_1}, \dots , v_{a_t} \}$ is a good spanner for the collection of $v_i$.  For any $v_i$, we can first write
\[
\mY^\top v_i = u_i = c_1 u_{a_1} + \dots + c_t u_{a_t} 
\]
where $|c_1|, \dots , |c_t| \leq 2$.  Now
\[
\begin{split}
\norm{v_i - (c_1 v_{a_1} + \dots + c_t v_{a_t} )}_{1} &\leq \sqrt{d}\norm{v_i - (c_1 v_{a_1} + \dots + c_t v_{a_t} )}_2 \\ & = \sqrt{d} \norm{(\mI - \mY \mY^\top ) (v_i - (c_1 v_{a_1} + \dots + c_t v_{a_t} ))}_2  \\ & \leq  (2t + 1)\sqrt{d} \norm{(\mI - \mY \mY^\top) \mM}_{\op} \\ &\leq 3t \sqrt{d} \sigma_{t+1} \\ &\leq 3\gamma s \sqrt{nd}
\end{split}
\]
ad thus the set $\{v_{a_1}, \dots , v_{a_t} \}$ is a $(2, 3\gamma s\sqrt{nd})$ spanner.  It is clear that the overall runtime is $\poly(n,d) \cdot \log(1/\gamma)$.   
\end{proof}

\subsection{Dimensionality Reduction}

We will also need a type of dimensionality for distributions over exponentially large domains.  For distributions $\calD_1, \dots , \calD_m$ over a large domain say $[N]$, we can represent them as $N$-dimensional vectors.  However, we need a more succinct representation.  Below we show how with only polynomially many samples and pdf access to the distributions, we can construct a succinct representation that approximately preserves important properties of the original distributions. We use the following notation for translating between distributions and vectors of their densities.

\begin{definition}
Let $\calD$ be a distribution on a finite set of elements $[N]$. For $a \in [N]$, we write $\calD[a]$ for the density function of $\calD$ at $a$.  We define the vector $\vec_{\calD} \in \R^N$ to be $\vec_{\calD} = (\calD[1], \dots , \calD[N])$.  For a multiset $\calX$ of elements in $[N]$, we define $\vec_{\calD}[\calX] \in \R^{|\calX|}$ to be a vector whose entries are indexed by elements $a \in \calX$ and the values are equal to $\calD[a]$.
\end{definition}

\begin{definition}\label{def:representative}
Given distributions $\calD_1, \dots , \calD_m$ on a set of elements, say $[N]$, we say vectors $u_1, \dots , u_m \in \R^d$ for some dimension $d$ are $(r, \gamma)$-representative for $\calD_1, \dots , \calD_m$ if for all coefficients $c_1, \dots , c_m \in \R$ with $|c_i| \leq r$  such that at most $r$ of them are nonzero,
\[
\left\lvert \norm{c_1 \vec_{\calD_1} + \dots + c_m \vec_{\calD_m}}_1 - \norm{c_1u_1 + \dots + c_mu_m}_1 \right\rvert \leq \gamma \,.
\]
\end{definition}

\begin{definition}\label{def:KL-preserving}
Given distributions $\calD_1, \dots , \calD_m$ on a set of elements, say $[N]$, we say vectors $w, u_1, \dots , u_m \in \R^d$ for some dimension $d$ are $(r, \gamma, \tau)$ KL-preserving for $\calD_1, \dots , \calD_m$ if for all coefficients $c_1, \dots , c_m, c_1', \dots , c_m' \in \R$ with $|c_i| \leq r, |c_i'| \leq r/\tau$  such that at most $r$ of the $c_i$ and $r$ of the $c_i'$ are nonzero and for any constant $\tau^*$ with $\tau \leq \tau^* \leq 1$,
\[
\left\lvert \KL_{\geq \tau^*}(c_1\calD_1 + \dots + c_m \calD_m \| c_1' \calD_1 + \dots + c_m' \calD_m ) - \KL_{\geq \tau^* w}(c_1 u_1 + \dots + c_m u_m \| c_1' u_1 + \dots + c_m' u_m) \right\rvert \leq \gamma \,.
\]    
\end{definition}

The main result of this section is Lemma~\ref{lem:dimensionality-reduction}, where we show that for any distributions $\calD_1, \dots , \calD_m$, after sampling and appropriately re-normalizing, we can construct succinct representations that approximately preserve the TV distance and KL divergence between linear combinations of these distributions.  

\begin{lemma}\label{lem:dimensionality-reduction}
Let $\calD_1, \dots , \calD_m$ be distributions on a set of elements, say $[N]$.  Let $k \in \N$ and let $\calX_1, \dots , \calX_m$ be multisets where $\calX_i$ is obtained by drawing $k$ elements independently from $\calD_i$.  Let $\calX = \calX_1 \cup \dots \cup \calX_m$.  Define the distribution $\wh{\calD} = \frac{\calD_1 + \dots + \calD_m}{m}$.  Define the vectors $u_i = \frac{\vec_{\calD_i}[\calX]}{mk \vec_{\wh{\calD}}[\calX]}$ for $i \in [m]$ where the division is done entrywise and define $w = 1/(mk \vec_{\wh{\calD}}[\calX])$ (reciprocated entrywise).

Let $0 < \delta,\gamma, \tau < 1$ and $r \in \N$ be  be some parameters. 
If $k \geq 100mr^4 \log^4 (m/(\tau\delta\gamma)) /\gamma^2 $ then with $1 - \delta$ probability
\begin{itemize}
    \item The vectors $u_1, \dots , u_m$ are $(r,\gamma)$-representative for $\calD_1, \dots , \calD_m$
    \item The vectors $w, u_1, \dots , u_m$ are $(r,\gamma,\tau)$ KL-preserving for $\calD_1, \dots , \calD_m$
\end{itemize}

\end{lemma}
\begin{proof}
We begin by proving the first statement.  Consider a fixed set of coefficients $c_1, \dots , c_m$.  We will first imagine that the sets $\calX_1, \dots , \calX_m$ are sampled after fixing the coefficients and then we will union bound over a net over all possible choices for $c_1, \dots , c_m$.  We have
\[
\norm{c_1\vec_{\calD_1} + \dots + c_m \vec_{\calD_m}}_1 = \E_{a \sim \wh{\calD}}\left[\frac{|c_1 \calD_1[a] + \dots + c_m \calD_m[a] |}{\wh{\calD}[a]} \right] \,.
\]

Also the quantity inside the expectation on the RHS always has magnitude at most $rm$.  Now we can draw samples from $\wh{\calD}$ to approximate the RHS.  By a Chernoff bound, if we draw at least $100m^2r^4/\gamma^2 \cdot \log(m/(\delta\gamma))$ samples from the distribution $\wh{\cal{D}}$ \--- this corresponds to taking $k \geq 100mr^4/\gamma^2 \cdot \log(m/(\delta\gamma)) $, then with probability at least $1 - (\delta\gamma/(10rm))^{10r}$, the empirical estimate of the RHS is within $\gamma/2$ of the true value i.e.
\[
\left\lvert \norm{c_1 \vec_{\calD_1} + \dots + c_m \vec_{\calD_m}}_1 - \norm{c_1u_1 + \dots + c_mu_m}_1 \right\rvert \leq \frac{\gamma}{2}
\]
which is exactly the inequality that we want to show.

Since the failure probability of the above is $(\delta\gamma/(10rm))^{10r}$, we can union bound over a $\gamma/(10rm)^3$-net of all of the possible choices of $c_1, \dots , c_m$.  Call the net $\calN$.   Then for any possible $c_1, \dots , c_m$ such that $|c_i| \leq r$ and at most $r$ of the $c_i$ are nonzero, there is some element of the net, say $(c_1', \dots , c_m') \in \calN$, such that $\max_{i \in [m]} |c_i  - c_i'| \leq \gamma/(10rm)^3$.  From the union bound over the net, we know that
\[
\left\lvert \norm{c_1'\vec_{\calD_1} + \dots + c_m' \vec_{\calD_m}}_1 - \norm{c_1'u_1 + \dots + c_m'u_m}_1 \right\rvert \leq \frac{\gamma}{2}
\]
and thus
\[
\left\lvert \norm{c_1\vec_{\calD_1} + \dots + c_m \vec_{\calD_m}}_1 - \norm{c_1u_1 + \dots + c_mu_m}_1 \right\rvert \leq \gamma 
\]
as desired.
\\\\
We prove the second statement in a similar way.  First consider a fixed choice of $\tau^*$.  Consider fixed coefficients $c_1, \dots , c_m, c_1' ,\dots , c_m'$.  Then 
\[
\begin{split}
&\KL_{\geq \tau^*}(c_1\calD_1 + \dots + c_m \calD_m \| c_1' \calD_1 + \dots + c_m' \calD_m ) \\ &= \E_{a \sim \wh{\calD}}\left[ \frac{c_1 \calD_1[a] + \dots + c_m \calD_m[a] }{\wh{\calD}[a]} \cdot  \log \max((c_1 \calD_1[a] + \dots + c_m \calD_m[a]  , \tau^* ) \right] 
\\ & - \E_{a \sim \wh{\calD}}\left[ \frac{c_1 \calD_1[a] + \dots + c_m \calD_m[a] }{\wh{\calD}[a]}  \cdot  \log \max((c_1' \calD_1[a] + \dots + c_m' \calD_m[a]), \tau^*)  \right] \\ &= mk \E_{a \sim \wh{\calD}}\left[ \frac{c_1 \calD_1[a] + \dots + c_m \calD_m[a] }{mk \wh{\calD}[a]} \cdot   \log \max\left( \frac{c_1 \calD_1[a] + \dots + c_m \calD_m[a]}{mk \wh{\calD}[a]} , \frac{\tau^*}{mk\wh{\calD}[a]} \right) \right] 
\\ & -  mk \E_{a \sim \wh{\calD}}\left[ \frac{c_1 \calD_1[a] + \dots + c_m \calD_m[a] }{mk \wh{\calD}[a]}  \cdot  \log \max\left( \frac{c_1' \calD_1[a] + \dots + c_m' \calD_m[a]}{mk \wh{\calD}[a]} , \frac{\tau^*}{mk\wh{\calD}[a]} \right) \right] \,.
\end{split}
\]
Also for all possible choices of $a \sim \wh{\calD}$, the RHS always has magnitude at most $10rm \log (m/\tau)$.  Now, as before in the proof of the first statement, by a Chernoff bound, this implies that for fixed $c_1, \dots , c_m, c_1' , \dots , c_m'$ and also $\tau^*$, with probability $1 - (\delta\gamma \tau/(10rm))^{10r}$, we have
\[
\left\lvert \KL_{\geq \tau^*}(c_1\calD_1 + \dots + c_m \calD_m \| c_1' \calD_1 + \dots + c_m' \calD_m ) - \KL_{\geq \tau^* w}(c_1 u_1 + \dots + c_m u_m \| c_1' u_1 + \dots + c_m' u_m) \right\rvert \leq \frac{\gamma}{2} \,.
\]
Again, as before, we next union bound over a $\gamma \tau/(10rm)^3$-net over all possible choices of $c_1, \dots,  c_m, c_1', \dots , c_m'$ and also $\tau^*$ and we deduce that 
\[
\left\lvert \KL_{\geq \tau^*}(c_1\calD_1 + \dots + c_m \calD_m \| c_1' \calD_1 + \dots + c_m' \calD_m ) - \KL_{\geq \tau^* w}(c_1 u_1 + \dots + c_m u_m \| c_1' u_1 + \dots + c_m' u_m) \right\rvert \leq \gamma
\]
for all choices of $c_1, \dots,  c_m, c_1', \dots , c_m', \tau^*$, as desired.
\end{proof}

The following fact will also be useful as we won't have pdf access to the exact distribution that we get samples from but instead to a distribution that is close in TV distance.

\begin{lemma}\label{lem:perturbation}
Let $\calD_1, \dots , \calD_m$ be distributions on a set of elements, say $[N]$.  Let $\calD_1', \dots , \calD_m'$ be distributions such that $d_{\TV}(\calD_i, \calD_i') \leq \eps$ for all $i \in [m]$.  Let $k \in \N$ and let $\calX_1, \dots , \calX_m$ be multisets where $\calX_i$ is obtained by drawing $k$ elements independently from $\calD_i'$.  With probability $1 - 2km^2\sqrt{\eps}$, we have for all $i \in [m]$ and $a \in \calX_i$
\[
|\calD_{j}[a] - \calD_j'[a]| \leq \sqrt{\eps}\calD_i[a]
\]
for all $j \in [m]$.   
\end{lemma}
\begin{proof}
Fix an $i \in [m]$.  For $j \in [m]$, let $\calZ_{j,i} \subset [N]$ be the set of elements $a \in [N]$ such that 
\[
|\calD_{j}[a] - \calD_j'[a]| \geq \sqrt{\eps} \calD_i[a] \,.
\]
Note that by assumption,
\[
\sum_{a \in \calZ_{j,i}}|\calD_{j}[a] - \calD_j'[a]| \leq  \eps
\]
and thus we must have
\[
\Pr_{a \sim \calD_i}[a \in \calZ_{j,i} ] = \sum_{a \in \calZ_{j,i}}\calD_i[a] \leq \sqrt{\eps} 
\]
which also implies $\Pr_{a \sim \calD_i'}[a \in \calZ_{j,i} ] \leq 2\sqrt{\eps}$.  Now we can union bound this over all choices of $i,j$ and all of the samples drawn to get that the overall failure probability is at most $2km^2\sqrt{\eps}$ as desired.
\end{proof}

\section{Learning Algorithm}\label{sec:learning-main}

In this section, we present our learning algorithm.  In light of Lemma~\ref{lem:pdf-estimation}, throughout our learning algorithm, we will assume that we have sample/conditional query and pdf access to a distribution $\wh{\Hm}$ that is conditionally close to the unknown distribution $\Hm$.  We will only use conditional queries to $\Hm$ to simulate access to  $\wh{\Hm}$ and almost all of our reasoning will be with respect to the distribution $\wh{\Hm}$.

\subsection{Finding a Spanner for the State Space}\label{sec:learning-spanner}

Recall that by Fact~\ref{fact:low-rank}, for a fixed $t$ with $t < T$, the (vectorized) distributions $\{ \Pr_{\Hm}[\cdot | h] \}_{h \in \calO^t}$ all lie in some $S$-dimensional subspace.  Thus, by Fact~\ref{fact:spanner}, there exists a barycentric spanner consisting of $S$ elements corresponding to some $S$ histories.  Since $\wh{\Hm}$ is conditionally close to $\Hm$, the (vectorized) distributions $\{ \Pr_{\wh{\Hm}}[\cdot | h] \}_{h \in \calO^t}$ are all close to some $S$-dimensional subspace.   The first step in our algorithm will be to, for each $t$, compute a set $\calB \subseteq \calO^t$   of $S$ histories such that $\{ \Pr_{\wh{\Hm}}[\cdot | h] \}_{h \in \calB}$ are an approximate spanner for $\{ \Pr_{\wh{\Hm}}[\cdot | h] \}_{h \in \calO^t}$.  The main algorithm in this section is Algorithm~\ref{alg:build-spanner} and we analyze it in Lemma~\ref{lem:spanner-guarantees}.

First, we will need a few preliminaries.  It will be convenient to introduce the following notation for arranging possible histories and futures into a matrix.

\begin{definition}
For a subset $\calH \subseteq \calO^t$ for some $t$ and subset $\calX \subseteq \calO^{T - t}$, let $\mM^{(t)}[\calH, \calX]$ be the matrix with rows indexed by elements $h \in \calH$ and columns indexed by elements $x \in \calX$ and entries equal to $\Pr_{\wh{\Hm}}[x | h]$.  When either $\calX$ or $\calH$ is the full set, we may write $\mM^{(t)}[\calH, :]$ or $\mM^{(t)}[: , \calX]$
\end{definition}

We will repeatedly make use of the primitive in Algorithm~\ref{alg:build-vectors} for taking histories $h_1, \dots , h_s \in \calO^t$ and outputting vectors $u_1, \dots , u_s$ that are succinct representations (in the sense of Lemma~\ref{lem:dimensionality-reduction}) of the distributions $\Pr_{\wh{\Hm}}[\cdot | h_1], \dots , \Pr_{\wh{\Hm}}[\cdot | h_s]$.

\begin{algorithm}[ht!]
\caption{Building Vectors}
\begin{algorithmic}[1]
\State \textbf{Input:} Sample, conditional query, and exact pdf access to a distribution $\wh{\Hm}$ over $\calO^T$
\State \textbf{Input:} Parameter $t < T$, strings $h_1, \dots , h_s \in \calO^t $
\State \textbf{Input:} Parameter $k \in \N$
\For{$i \in [s]$}
\State Sample $k$ futures $f_1, \dots , f_k \in \calO^{T - t}$ from $\Pr_{\wh{\Hm}}[\cdot | h_i]$
\State Set $\calX_i = \{f_1, \dots , f_k \}$
\EndFor
\State Set $\calX = \calX_1 \cup \dots \cup \calX_s$ (as a multiset)
\For{$i \in [s]$}
\State Compute $\Pr_{\wh{\Hm}}[f|h_i]$ for all $f \in \calX$ 
\State Set $v_i \in \R^{|\calX|}$ as $v_i = \mM^{(t)}[ h_i , \calX]$
\EndFor
\State Compute $w = 1/(k(v_1 + \dots + v_s))$ (reciporicated entrywise)
\State Compute $u_i = v_i \odot w$ (multiplication entrywise)
\State \textbf{Output:} $\calX, \{ u_1, \dots , u_s \}, w$
\end{algorithmic}
\label{alg:build-vectors}
\end{algorithm}

\begin{fact}\label{fact:bounded-output-vectors}
Whenever we run Algorithm~\ref{alg:build-vectors}, the output vectors $u_1, \dots , u_s$ have nonnegative entries and satisfy $\norm{u_1}_{\infty}, \dots , \norm{u_s}_{\infty} \leq 1$,
\end{fact}
\begin{proof}
This follows immediately from the construction in Algorithm~\ref{alg:build-vectors}.    
\end{proof}

Now we are ready to present the algorithm for building the spanners and its analysis.  For technical reasons that will be important later, when computing a spanner for histories of length $t+1$, the algorithm takes as input some strings $h_1^{(t)}, \dots , h_s^{(t)}$ of length $t$.  The goal will be to output a list of strings of length $t+1$, say $\calB \subseteq \calO^{t+1}$, such that the corresponding vectors $\{\mM^{(t+1)}[h,:] \}_{h \in \calB}$ are an approximate spanner for the distribution of vectors $\{\mM^{(t+1)}[h,:] \}_{h \sim \wh{\Hm}[:t+1]}$ and they also span $\{\mM^{(t+1)}[h_i^{(t)} \vee o,:] \}_{h \in \calB}$ for all $i \in [s], o \in \calO$.  Note that the first condition is the natural requirement for a spanner and the second will be important later when we compute ``transitions" between the length-$t$ histories and length-$t+1$ histories.

\begin{algorithm}[ht!]
\caption{Building Spanners}
\begin{algorithmic}[1]
\State \textbf{Input:} Sample, conditional query, and exact pdf access to a distribution $\wh{\Hm}$ over $\calO^T$
\State \textbf{Input:} Parameter $t < T$, strings $h_1^{(t)}, \dots , h_s^{(t)} \in \calO^t $ 
\State \textbf{Input:} Parameters $S \in \N, S \geq s$, $\gamma, \delta > 0$
\State Define $\calA \subseteq \calO^{t + 1}$ as $\calA = \{ h_i^{(t)} \vee o \}_{i \in [s], o \in \calO}$
\State Draw $N = \poly(S, O, T, 1/\gamma, \log(1/\delta)) $ histories of length $t+1$ from $\wh{\Hm}[:t+1]$, say $h'_1, \dots , h'_N$
\State Set $\calA' \leftarrow \calA \cup \{h'_1, \dots , h'_N \}$
\State Run Algorithm~\ref{alg:build-vectors} on set $\calA'$ with parameters $t+1, k = N \cdot \poly(S, O, T, 1/\gamma, \log(1/\delta))$
\State Let the result be $\calX , \{u_h \}_{h \in \calA'}, w$ where $\calX \subseteq \calO^{T - t - 1}, u_h \in \R^{|\calX|}$
\State Run Algorithm~\ref{alg:robust-spanner} on $\{u_h \}_{h \in \calA'}$ with parameters $S, \gamma/(100S k^2)$ and let the output be $\calB \subseteq \calA'$

\State \textbf{Output:} $\calB$
\end{algorithmic}
\label{alg:build-spanner}
\end{algorithm}

We now prove guarantees on the spanning set $\calB$ found by Algorithm~\ref{alg:build-spanner}.

\begin{lemma}\label{lem:spanner-guarantees}
Assume that $\wh{\Hm}$ is $\eps$ conditionally close to a rank $S$ distribution $\Hm$ where
\[
\eps \leq \frac{1}{\poly(S,O,T, 1/\gamma, \log(1/\delta))}
\]
for some sufficiently large polynomial.  With probability at least   $1 - \delta - \eps^{0.4} $,  if we run Algorithm~\ref{alg:build-spanner} for arbitrary input strings $h_1^{(t)}, \dots , h_s^{(t)}$, the output satisfies the following conditions
\begin{itemize}
\item  $|\calB| \leq S$ and the vectors $\{u_h \}_{h \in \calB}$ form a $(2, \gamma)$ spanner for $\{u_h \}_{h \in \calA'}$
\item The rows of $\{\mM^{(t+1)}[h,:] \}_{h \in \calB}$ form a $(2, S\gamma)$-spanner for the rows of $\{ \mM^{(t+1)}[h,:] \}_{h \in \calA'}$
\item The rows of $\{\mM^{(t+1)}[h,:] \}_{h \in \calB}$ form a $(1 - \gamma, 2S, 2S\gamma)$-spanner for the distribution of vectors $\{ \mM^{(t+1)}[h,:] \}_{h \sim \wh{\Hm}[:t+1]}$
\end{itemize}
\end{lemma}
\begin{proof}
For $h \in \calO^{t + 1}$, let $\nu_h, \nu'_h$ be vectors with entries indexed by $x \in \calO^{T - t - 1}$ and entries given by $\Pr_{\Hm}[ x | h]$ and $\Pr_{\wh{\Hm}}[ x|h ]$ respectively.  First, note that the vectors $\{\nu_h\}$ are contained in an $S$-dimensional space by Fact~\ref{fact:low-rank}.  Thus, by Lemma~\ref{lem:perturbation}, and the setting of $\eps$ sufficiently small, with probability at least  $1 - \eps^{0.4}$, there is an $S$-dimensional subspace, say $\mV$, such that all of $\{ u_h\}_{h \in \calA'}$ have projection at most $\eps^{0.1}$ onto $\mV^{\perp}$.  Assuming that this holds, all of the hypotheses of Lemma~\ref{lem:compute-spanner-full} are satisfied (note the condition about the norms of the input vectors is satisfied by the construction in Algorithm~\ref{alg:build-vectors}).  We then get that the execution of Algorithm~\ref{alg:build-spanner} successfully finds a set $\{ u_h \}_{h \in \calB}$ with $|\calB| \leq S$ that is a $(2, \gamma)$ spanner for $\{u_h \}_{h \in \calA'}$.

For each $h' \in \calA'$, we can use the fact that $\{u_h \}_{h \in \calB}$ is a  $(2, \gamma)$ spanner for $\{u_h \}_{h \in \calA'}$ to find a linear combination $\{ c_{h', h}\}_{h \in \calB}$ with $|c_{h',h}| \leq 2$ for all $h \in \calB$ such that
\begin{equation}\label{eq:approx1}
\norm{u_{h'} - \sum_{h \in \calB} c_{h',h} u_h}_1 \leq \gamma\,.
\end{equation}
Next, by Lemma~\ref{lem:dimensionality-reduction}, with probability at least $1 - 0.1\delta$, the vectors $\{u_h \}_{h \in \calA'}$ are $(10S, 0.1\gamma)$-representative for the distributions $\{\Pr_{\wh{\Hm}}[ \cdot | h ] \}_{h \in \calA'}$.  This then implies
\[
\norm{\mM^{(t+1)}[h', :] - \sum_{h \in \calB} c_{h',h} \mM^{(t+1)}[h, :]}_1 \leq S\gamma 
\]
which gives the second condition.

Now we prove the last condition.  Recall that the vectors $\{\nu_h \}$ are vectors in $\R^{O^{T - t - 1}}$ that are contained in some $S$-dimensional subspace.  Let $\calP$ be the distribution of these vectors for $h \sim \wh{\Hm}[:t+1]$.  By Lemma~\ref{lem:distribution-spanner} (since these vectors live in an $S$-dimensional subspace), with probability at least $1 - 0.1\delta$, there is a subset $\wh{\calB} \subseteq \calA'$ with $|\wh{\calB}| \leq S$, such that the vectors $\{\nu_h \}_{h \in \wh{\calB}}$ are a $( 1 - \gamma, 1, 0)$-spanner for this distribution $\calP$.  For each $h' \in \wh{\calB}$, we can again use \eqref{eq:approx1} (since $\wh{\calB} \subseteq \calA'$).  We can also again use Lemma~\ref{lem:dimensionality-reduction} to get that the vectors $\{u_h \}_{h \in \calA'}$ are $(10S, 0.1\gamma)$-representative for the distributions $\{\Pr_{\wh{\Hm}}[ \cdot | h ] \}_{h \in \calA'}$ and thus by the conditional closeness of $\wh{\Hm}$ and $\Hm$ and Claim~\ref{claim:nextchar-to-TV}, they are  $(10S, 0.2\gamma)$-representative for $\{\Pr_{\Hm}[ \cdot | h ] \}_{h \in \calA'}$.  Thus, \eqref{eq:approx1} implies that for all $h' \in \wh{\calB}$, there exist coefficients $\{ c_{h', h}\}_{h \in \calB}$ with $|c_{h',h}| \leq 2$ for all $h \in \calB$ such that
\begin{equation}\label{eq:approx2}
\norm{\nu_{h'}- \sum_{h \in \calB} c_{h',h} \nu_h}_1 \leq 2\gamma
\end{equation}
since $|\calB| \leq S$.   Next, recall that $\{\nu_{h} \}_{h \in \wh{\calB}}$ are a $( 1 - \gamma, 1, 0)$-spanner for the distribution $\calP$ of $\{\nu_h \}$ for $h \sim \wh{\Hm}[:t+1]$.  Thus, for $h'' \sim \wh{\Hm}[:t+1]$, with $1 - \gamma$ probability, there are coefficients $\{c'_{h'} \}_{h' \in \wh{\calB}}$ with $|c'_{h'}| \leq 1$ such that 
\[
\nu_{h''} = \sum_{h' \in \wh{\calB}} c'_{h'} \nu_{h'} \,.
\]
Finally, we can combine this with \eqref{eq:approx2} (which we can apply for all $h' \in \wh{\calB}$) and the closeness between $\wh{\Hm}$ and $\Hm$ to get that there are coefficients $\{c_h \}_{h \in \calB}$ with $|c_h| \leq 2S$ such that 
\[
\norm{\nu'_{h''} - \sum_{h \in \calB} c_h \nu_{h}' }_1 \leq 2S\gamma
\]
which gives the second statement.  Note that the overall failure probability for all of the statements is at most $\delta + \eps^{0.4}$.
\end{proof}

\subsection{Full Learning Algorithm}\label{sec:full-learning}

For our full learning algorithm, we will apply Algorithm~\ref{alg:build-spanner} to build a sequence of spanning sets $\calH^{(1)}, \dots , \calH^{(T-1)}$ for each history-length.  We then need to learn the ``transitions" between these spanning sets.  Rather than explicitly learning these transitions, we instead learn a representation of the distribution of the form described below where the transitions are implicit \--- the transitions are only computed in Section~\ref{sec:sampling-alg} when we actually want to sample from the learned distribution.  The representation of the distribution consists of:
\begin{itemize}
\item  Subsets $\calH^{(t)} \subseteq \calO^t$ for all $0 \leq t \leq T - 1$
\item Matrices $\mP^{(t)} \in \R^{|\calH^{(t)}| \times O}$ for all $0 \leq t  \leq T - 1$ containing the next character probabilities for each of the strings in $\calH^{(t)}$
\item For each $0 \leq t \leq T$, a collection of vectors $u_{h^{(t)}}$ for all $h^{(t)} \in \calH^{(t)} \cup \{ \calH^{(t-1)} \vee o \}_{o \in \calO} $ that are succinct representations of the distributions $\Pr_{\wh{\Hm}}[\cdot | h^{(t)}]$ obtained from Algorithm~\ref{alg:build-vectors}
\item For each $0 \leq t \leq T$, we also maintain an index set $\calX^{(t)}$  and vector $w^{(t)}$ (these are also obtained from Algorithm~\ref{alg:build-vectors})
\end{itemize}

\noindent Algorithm~\ref{alg:learning} gives a full description of how we compute this representation.

\begin{algorithm}[ht!]
\caption{Next Character Probabilities}
\begin{algorithmic}[1]
\State \textbf{Input:} Sample, conditional query, and exact pdf access to a distribution $\wh{\Hm}$ over $\calO^T$
\State \textbf{Input:} Parameter $t < T$, strings $h_1^{(t)}, \dots , h_s^{(t)} \in \calO^t $ 
\For {$o \in \calO$}
\State Compute probabilities $p_{1,o} = \Pr_{\wh{\Hm}}[o| h_1^{(t)}], \dots , p_{s,o} = \Pr_{\wh{\Hm}}[o| h_s^{(t)}] $
\EndFor
\State Let $\mP^{(t)} \in \R^{s \times O}$ be the matrix whose entries are given by $\{p_{i,o}\}_{i \in [s], o \in \calO}$
\State \textbf{Output:} $\mP^{(t)}$
\end{algorithmic}
\label{alg:next-char}
\end{algorithm}

\begin{algorithm}[ht!]
\caption{Full Learning Algorithm}
\begin{algorithmic}[1]
\State \textbf{Input:} Sample, conditional query, and exact pdf access to a distribution $\wh{\Hm}$ over $\calO^T$
\State \textbf{Input:} Parameters $T, S \in \N$, $\eta, \delta > 0$
\State Initialize $\calH^{(0)}$ to be a set consisting of only the empty string
\State Set $k = \poly(S, O, T, 1/\eta, \log(1/\delta))$
\State Set $\gamma = 1/\poly(k)$
\For{$t = 0,1, \dots , T - 1$}
\State Run Algorithm~\ref{alg:next-char} on $\calH^{(t)}$ and let the output be $\mP^{(t+1)}$
\State Run Algorithm~\ref{alg:build-spanner} with input $\calH^{(t)}$ and parameters $t, S, \gamma, \delta$
\State Let the output be $\calH^{(t+1)}$
\State Set $\calB^{(t+1)} = \calH^{(t+1)} \cup  \{ \calH^{(t)}  \vee o \}_{o \in \calO}$
\State Run Algorithm~\ref{alg:build-vectors} on set $\calB^{(t+1)}$ with parameters $t+1, k$ \label{line:build-vecs}
\State Let the result be $\calX^{(t+1)}, \{ u_h\}_{h \in \calB^{(t+1)}}, w^{(t+1)}$
\EndFor
\State \textbf{Output:} $\{\mP^{(t)} , \calH^{(t)} , \calX^{(t)},  \{ u_{h^{(t)}}\}_{h^{(t)} \in \calB^{(t)}}, w^{(t)} \}_{t \in [T]}$

\end{algorithmic}
\label{alg:learning}
\end{algorithm}

After running Algorithm~\ref{alg:learning}, we obtain a description of the distribution in terms of 
\[
\{\mP^{(t)} , \calH^{(t)}, \calX^{(t)}, \{ u_h\}_{h \in \calB^{(t)}}, w^{(t)} \}_{t \in [T]} \,.
\]
Note that for all $t$, the set $\calB^{(t)}$ is just defined by $\calB^{(t)} = \calH^{(t)} \cup \{ \calH^{(t-1)} \vee o \}_{o \in \calO}$.  In the remainder of this section, we prove that the parameters learned above satisfy certain properties (see Lemma~\ref{lem:recovery-guarantees}).  These properties will then be used in Section~\ref{sec:sampling-alg} to argue that we can sample from the learned description to get a distribution that is close to the original distribution $\wh{\Hm}$.  We begin with a few definitions.

\begin{definition}
We say a string $h \in \calO^t$ for $t < T$ is $\gamma$-\textit{representable} by $\calH \subseteq \calO^t$ if there exists some vector $y \in \R^{|\calH|}$ with $\norm{y}_{\infty} \leq 2S$ such that 
\[
\norm{\mM^{(t)}[h,:] - y^\top \mM^{(t)}[\calH, :] }_1 \leq 2S \gamma \,.
\]
\end{definition}

\begin{definition}
We say a string $h \in \calO^t$ for $t < T$ is $(c,\gamma)$-positively \textit{representable}  by $\calH \subseteq \calO^t$ and $\calX \subseteq O^{T - t}$ if there exists some vector $y \in \R^{|\calH|}$ with $\norm{y}_{\infty} \leq 2S$ such that
\begin{itemize}
\item $\norm{\mM^{(t)}[h,:] - y^\top \mM^{(t)}[\calH, :] }_1 \leq 2S \gamma $
\item $y^\top \mM^{(t)}[\calH, \calX]$ has all entries at least $c$
\end{itemize}
\end{definition}

\noindent The main lemma of this section is stated below.

\begin{lemma}\label{lem:recovery-guarantees}
Assume that $\wh{\Hm}$ is $\eps$ conditionally close to a rank $S$ distribution $\Hm$ where
\[
\eps \leq \frac{1}{\poly(S,O,T, 1/\eta, \log(1/\delta))}
\]
for some sufficiently large polynomial.  For any parameter $c \geq e^{-10 (OTS)^2/\eta}$, in the execution of Algorithm~\ref{alg:learning}, with probability $1 - \delta - \gamma^{0.1}$, we have the following properties:
\begin{itemize}
\item For any $t \leq T$, $|\calH^{(t)}| \leq S$.
\item For any $t \leq T$ and $h \sim \wh{\Hm}[:t]$, the string $h$ is $(c, 4S\sqrt{\gamma})$-positively representable by $\calH^{(t)}, \calX^{(t)}$ with at least $ 1 - \gamma^{0.1} - \frac{c \cdot O^T}{\gamma}$ probability
\item For all $t \in [T]$, the vectors $ \{ u_h^{(t)} \}_{h \in \calB^{(t)} }$ are $(10S, \eta)$-representative for the distributions $\{\Pr_{\wh{\Hm}}[\cdot | h] \}_{h \in \calB^{(t)}}$
\item   For all $t \in [T]$, the vectors  $ w^{(t)}, \{ u_h^{(t)} \}_{h \in \calB^{(t)} }$ are $(10S, \eta, c)$ KL-preserving for the distributions \\ $\{\Pr_{\wh{\Hm}}[\cdot | h] \}_{h \in \calB^{(t)}}$    
\end{itemize}
\end{lemma}
\begin{remark}
We think of the parameters as being set such that $\eps \ll \poly(1/\gamma)^{-1}, \gamma \ll \poly(1/\eta)^{-1}$.
\end{remark}

\begin{proof}
We apply Lemma~\ref{lem:spanner-guarantees} \--- with probability $1 - \delta/T - \eps^{0.4}$, the guarantees of the lemma hold each time we execute Algorithm~\ref{alg:build-spanner}. This immediately implies the first statement that we want to show.  For the second statement, Lemma~\ref{lem:spanner-guarantees} also gives us that over the execution of the algorithm, for all $t < T$, the the vectors $\{ \mM^{(t)}[h, :] \}_{h \in \calH^{(t)}}$ form a $(1 - \gamma, 2S, 2S\gamma)$-spanner for the distribution of vectors $\{ \mM^{(t)}[h,:] \}_{h \sim \wh{\Hm}[:t]}$.  Thus, for $h \sim \wh{\Hm}[:t]$, with $1 - \gamma$ probability, there exists some vector $y \in \R^{|\calH|}$ with $\norm{y}_{\infty} \leq 2S$ such that
\begin{equation}\label{eq:l1-bound1}
\norm{\mM^{(t)}[h,:] - y^\top \mM^{(t)}[\calH^{(t)}, :] }_1 \leq 2S \gamma \,.
\end{equation}
Let $y'$ be obtained by increasing all entries of $y$ by $\sqrt{\gamma}$.  Since $|\calH^{(t)}| \leq S$, we have
\[
\norm{\mM^{(t)}[h,:] - {y'}^\top \mM^{(t)}[\calH^{(t)}, :] }_1 \leq 4S \sqrt{\gamma} \,.
\]
Now first consider sampling $\calX^{(t)}$ after $\calH^{(t)}$ is fixed.  Now let $\calZ \subseteq \calO^{T - t}$ be the set of all strings $x$ such that 
\begin{equation}\label{eq:def-bad}
y^\top \mM^{(t)}[\calH^{(t)}, x] \leq c - \sqrt{\gamma} \sum_{h \in \calH^{(t)}}\Pr_{\wh{\Hm}}[x|h] \,.
\end{equation}
As long as $\calX^{(t)}$ doesn't contain any of these elements, then $y'$ would give us a $(c,4S\sqrt{\gamma})$-positive representation of $h$ by $\calH^{(t)}, \calX^{(t)}$.  Note that
\[
-\left( \sum_{x \in \calZ} y^\top \mM^{(t)}[\calH^{(t)}, x] \right) \leq \norm{\mM^{(t)}[h,:] - y^\top \mM^{(t)}[\calH^{(t)}, :] }_1 
\]
and thus we must have
\[
\sum_{x \in \calZ} y^\top \mM^{(t)}[\calH^{(t)}, x]  \geq -2S \gamma
\]
and using \eqref{eq:def-bad}, this rearranges into
\[
\sum_{x \in \calZ} \sum_{h \in \calH^{(t)}}\Pr_{\wh{\Hm}}[x|h] \leq \frac{c \cdot |\calZ| + 2S\gamma}{\sqrt{\gamma}} \leq \frac{c \cdot O^T + 2S\gamma}{\sqrt{\gamma}} \,.
\]
Also recall Lemma~\ref{lem:spanner-guarantees} implies that for any $h \in \calB^{(t)}$,
\[
\sum_{x \in \calZ} \Pr_{\wh{\Hm}}[x | h] \leq S\gamma + 2 \sum_{x \in \calZ}\sum_{h \in \calH^{(t)}}  \Pr_{\wh{\Hm}}[x | h] \leq \frac{3(c \cdot O^T + 2S\gamma)}{\sqrt{\gamma}}\,.
\]
Now we have
\[
\Pr[\calX^{(t)} \cap \calZ = \emptyset] \geq 1 - |\calB^{(t)}| k \cdot \frac{3(c \cdot O^T + 2S\gamma)}{\sqrt{\gamma}} \geq 1 - \gamma^{0.3} - \frac{c \cdot O^T}{\gamma^{0.6}}
\]
and if this happens, then $h$ is $(c, 4S\sqrt{\gamma})$-positively representable by $\calH^{(t)}, \calX^{(t)}$.  Thus, so far we have shown that over the randomness of $h \sim \wh{\Hm}[:t]$ and $\calX^{(t)}$,
\[
\Pr \left[h \text{ is } (c,4S\sqrt{\gamma}) \text{-positively representable by } \calH^{(t)}, \calX^{(t)} \right] \geq 1 - \gamma^{0.3} - \frac{c \cdot O^T}{\gamma^{0.6}} \,.
\]
Now by Markov's inequality, with probability at least $1 - \gamma^{0.1}$ over the choice of $\calX^{(t)}$, 
\[
\Pr_{h \sim \wh{\Hm}[:t]} \left[h \text{ is } (c,4S\sqrt{\gamma}) \text{-positively representable by } \calH^{(t)}, \calX^{(t)}| \calX^{(t)} \right] \geq  1 - \gamma^{0.1} - \frac{c \cdot O^T}{\gamma}
\]
and this proves the second of the desired statements.  

Finally, the last two statements follow from Lemma~\ref{lem:dimensionality-reduction}.  Combining the failure probabilities for all of the parts, the total failure probability is at most $\delta + \gamma^{0.1}$, and this completes the proof. 
\end{proof}

Later on, we will also need the following simple observation.

\begin{claim}\label{claim:ref-entries-small}
In the execution of Algorithm~\ref{alg:learning}, assume that $|\calH^{(t)}| \leq S$ for all $t \leq T$.  Then for any $c > 0$, with probability at least $1 -  (OS)^{2T} \sqrt{c}$, we have $\norm{w^{(t)}}_{\infty} \leq 1/\sqrt{c}$  for all $t \in [T]$.   
\end{claim}
\begin{proof}
Consider a  fixed $t \in [T]$.  The desired statement is failed only when in line~\ref{line:build-vecs} of Algorithm~\ref{alg:learning}, when we execute the subroutine Algorithm~\ref{alg:build-vectors}, we sample some $f \in \calO^{T - t}$ such that 
\[
\sum_{h^{(t)} \in \calB^{(t)}} \Pr[f|h^{(t)}] \leq \frac{\sqrt{c}}{k} \,.
\]
However, the probability of sampling such an element is at most 
\[
\frac{\sqrt{c}}{k} O^{T} k |\calB^{(t)}| \leq O^{T}(O + 1) S \sqrt{c} \,.
\]
Now union bounding over all $t \in [T]$ immediately gives the desired conclusion.
\end{proof}

\section{Sampling Procedure} \label{sec:sampling-alg}

After running Algorithm~\ref{alg:learning}, we have learned a set of parameters 
\[
\{\mP^{(t)} , \calH^{(t)}, \calX^{(t)}, \{ u_h\}_{h \in \calB^{(t)}}, w^{(t)} \}_{t \in [T]} 
\]
where recall $\calB^{(t)} = \calH^{(t)} \cup \{ \calH^{(t-1)} \vee o\}_{o \in \calO}$.  However, it is not yet clear how these parameters determine a distribution. 
In this section, we show how these parameters determine a distribution that we can efficiently sample from and argue that if the learned parameters satisfy the properties in Lemma~\ref{lem:recovery-guarantees}, then this distribution is close to $\wh{\Hm}$.

Throughout this section, we assume that we are given the global parameters $O,T,S$ and a target accuracy parameter $\eta$.  We make the following assumption on the accuracy of the learned parameters.  In light of Lemma~\ref{lem:recovery-guarantees} and Claim~\ref{claim:ref-entries-small}, we can show that this assumption holds with high probability as long as we ran Algorithm~\ref{alg:learning} on a distribution $\wh{\Hm}$ that is $\eps$-conditionally close to a rank $S$ distribution $\Hm$ for sufficiently small $\eps$.


\begin{assumption}\label{assume:good-params}
We have $\eta > 0$ satisfying $\eta < 1/\poly(O,T,S)$ for some sufficiently large polynomial.  Let $c = \eta^{10OTS}$.  The parameters
\[
\{\mP^{(t)} , \calH^{(t)}, \calX^{(t)}, \{ u_h\}_{h \in \calB^{(t)}}, w^{(t)} \}_{t \in [T]}
\]
satisfy the following properties:
\begin{itemize}
\item For any $t \leq T$, $|\calH^{(t)}| \leq S$.
\item For any $t \leq T$ and $h \sim \wh{\Hm}[:t]$, the string $h$ is $(2\sqrt{c}, \eta)$-positively representable by $\calH^{(t)}, \calX^{(t)}$ with at least $ 1 - \eta$ probability
\item For all $t \in [T]$, the vectors $  w^{(t)}, \{ u_h^{(t)} \}_{h \in \calB^{(t)} }$ have nonnegative entries and have the same dimensionality $d = \poly(S,O,T, 1/\eta)$  and satisfy $u_h^{(t)} = w^{(t)} \odot \mM^{(t)}[h, \calX^{(t)} ]$ (where $\odot$ denotes entrywise product)
\item For all $t \in [T]$, the vectors $ \{ u_h^{(t)} \}_{h \in \calB^{(t)} }$ are $(10S, \eta)$-representative for the distributions $\{\Pr_{\wh{\Hm}}[\cdot | h] \}_{h \in \calB^{(t)}}$
\item   For all $t \in [T]$, the vectors  $ w^{(t)}, \{ u_h^{(t)} \}_{h \in \calB^{(t)} }$ are $(10S, \eta, c^{T})$ KL-preserving for the distributions $\{\Pr_{\wh{\Hm}}[\cdot | h] \}_{h \in \calB^{(t)}}$    
\item $\norm{w^{(t)}}_{\infty} \leq 1/\sqrt{c}$ for all $t \in [T]$
\end{itemize}
\end{assumption}

We will also make the following assumption on the underlying distribution $\wh{\Hm}$.

\begin{assumption}\label{assume:positive-distribution}
The distribution $\wh{\Hm}$ is $c$-positive (recall Definition~\ref{def:positivity}).
\end{assumption}

Throughout the rest of this section, we will treat the parameters $\{\mP^{(t)} , \calH^{(t)}, \calX^{(t)}, \{ u_h\}_{h \in \calB^{(t)}}, w^{(t)} \}_{t \in [T]}$ as fixed.  We will describe the sampling procedure and then analyze it, proving that it samples from a distribution close to $\wh{\Hm}$ as long as the assumptions above hold.

The main sampling algorithm is Algorithm~\ref{alg:sampling-procedure}.  First, we define the following operation which rounds a vector to a point on the simplex corresponding to a probability distribution.

\begin{definition}\label{def:round}
For a vector $v \in \R^n$ and parameter $\tau > 0$, we define $\round_{\tau}(v)$ to be 
\[
\left(\frac{\max(v[1],\tau)}{C}, \dots , \frac{\max(v[n],\tau)}{C} \right)
\]
where $C = \sum_{i = 1}^n \max(v[i],\tau)$.  Note that the output of $\round()$ is always a valid probability distribution.
\end{definition}

At a high-level, Algorithm~\ref{alg:sampling-procedure} works as follows.  We sample a string $x$ one character at a time.  At each step $t < T$, we maintain a linear combination of the strings $h \in \calH^{(t)}$ that is supposed to approximate the string $x[:t] = o_1o_2 \dots o_t$ given by the first $t$ characters of $x$ in the following sense: we want
\[
\Pr_{\wh{\Hm}}[\cdot | x[:t]] \approx \sum_{h \in \calH^{(t)}} \alpha_h^{x[:t]} \Pr_{\wh{\Hm}}[\cdot | h]
\]
as distributions, for some coefficients $\{ \alpha_h^{x[:t]} \}_{h \in \calH^{(t)}}$. 
 Given this linear combination, and since $\mP^{(t)}$ gives us the next character probabilities for all of the strings in $\calH^{(t)}$, we know what the distribution for the next character should be given the first $t$ characters in $x[:t]$ so we simply sample the next character from this distribution.  If the next character is $o_{t+1}$, so $x[:t+1] = x[:t] \vee o_{t+1}$, then we can re-normalize the coefficients $\alpha_h^{x[:t]}$ to get coefficients $\wh{\alpha}_h^{x[:t]}$ such that 
\begin{equation}\label{eq:basis-change-informal}
 \Pr_{\wh{\Hm}}[\cdot | x[:t+1] ] \approx \sum_{h \in \calH^{(t)}} \wh{\alpha}_h^{x[:t]} \Pr_{\wh{\Hm}}[\cdot | h \vee o_{t+1}] \,.
\end{equation}
 The main difficulty is that we now want to rewrite the RHS as a linear combination of the distributions \\ $\{ \Pr_{\wh{\Hm}}[\cdot | h] \}_{h \in \calH^{(t+1)}}$.




Recalling the discussion in Section~\ref{sec:sample-intro}, we will solve a convex optimization problem, that roughly corresponds to projection in KL, for each ``change-of-basis" operation.  In particular, we will use the fact that the vectors $\{u_h\}_{h \in \calB^{(t+1)}}$ (where recall $\calB^{(t+1)} = \calH^{(t+1)} \cup \{ \calH^{(t)} \vee o \}_{o \in \calO}$) are succinct representations of $\Pr_{\wh{\Hm}}[\cdot | h]$.  Then once we have the coefficients $\wh{\alpha}_h^{x[:t]}$ in \eqref{eq:basis-change-informal}, to ``change the basis", we solve a convex optimization problem for coefficients $\alpha_h^{x[:t+1]}$ such that $|\alpha_h^{x[:t+1]}|$ are bounded and 
\[
\KL \left( \sum_{h \in \calH^{(t+1)}} \alpha_h^{x[:t+1]}u_{ h } \Bigg\| \sum_{h \in \calH^{(t)}} \wh{\alpha}_h^{x[:t]}u_{ h \vee o_{t+1}}  \right)
\]
is minimized.  Recall the crucial property of KL divergence that ensures the error only grows additively but not multiplicatively in each step is that projection in KL onto a convex set decreases the distance (in KL) \emph{to all} other points in the set  (see Fact~\ref{fact:KL-basic}).


For describing our algorithm formally, it will be convenient to define the following notation for arranging the rows $\{u_h \}_{h \in \calB^{(t)}}$ as the rows of a matrix.

\begin{definition}\label{def:rep-matrix}
Let $\mR^{(t)}$ be the matrix whose rows are given by $\{ u_h \}_{h \in \calH^{(t)}}$.  Let for $o \in \calO$, let $\mT^{(t)}_{o}$ be the matrix whose rows are given by $\{ u_h \}_{h \in \calH^{(t)} \vee o}$.
\end{definition}

\begin{algorithm}[ht!]
\caption{Sampling Procedure}
\begin{algorithmic}[1]
\State \textbf{Input:} $S, O, T \in \N, \eta > 0$
\State \textbf{Input:} Learned representation $\{\mP^{(t)} , \calH^{(t)}, \calX^{(t)}, \{ u_h\}_{h \in \calB^{(t)}}, w^{(t)} \}_{t \in [T]}$
\For{$t \in [T]$,  $o \in \calO$}
\State Let $\mR^{(t)}$ be the matrix whose rows are given by $\{ u_h \}_{h \in \calH^{(t)}}$
\State Let $\mT^{(t)}_{o}$ be the matrix whose rows are given by $\{ u_h \}_{h \in \calH^{(t)} \vee o}$.
\EndFor
\State Set $c = \eta^{-10 OTS}$
\State Set $x[:0]$ to be the empty string
\State Let $\boldsymbol{\alpha}^{x[:0]} = 1$
\For{$t = 0,1, \dots , T - 1$}
\State Compute probabilities $\{ p_o\}_{o \in \calO} = \round_{2c^{0.1}} ({\boldsymbol{\alpha}^{x[:t]}}^\top \mP^{(t)})$ \label{line:next-char}
\State Sample next character $o_{t+1} \in \calO$ with probabilities $\{ p_o\}_{o \in \calO}$ 
\State Set $x[:t+1] = x[:t] \vee o_{t+1}$ \label{line:append}
\State Let the entries in the column of $\mP^{(t)}$ indexed by $o_{t+1}$ be $\{ q^{o_{t+1}}_h \}_{h \in \calH^{(t)}}$ 
\State Define the vector $\bsy{\nu}^{x[:t+1]} =  \frac{1}{p_{o_{t+1}}} {\bsy{\alpha}^{x[:t]}}^\top \diag\left(\{ q^{o_{t+1}}_h \}_{h \in \calH^{(t)}}\right)$
\State Let 
\[
\mathbf{z}^{x[:t+1]} = {\bsy{\nu}^{x[:t+1]}}^\top \mT_{o}^{(t)} 
\]

\State \label{line:convex-opt} Compute vector $\boldsymbol{\alpha} = \{\alpha_{h} \}_{h \in \calH^{(t+1)}}$ such that
\[
\begin{split}
&|\alpha_{h}| \leq 3S \quad\quad  \forall h \in \calH^{(t+1)} \\
&{\bsy{\alpha}}^\top \mR^{(t+1)} \geq c^{T - t } w^{(t+1)} \text { entrywise} \\
& \langle {\bsy{\alpha}}^\top \mR^{(t+1)} , \mathbf{1} \rangle = 1 \text{ where } \mathbf{1} \text{ is the all ones vector} 
\end{split}
\]
\indent and which minimizes \label{step:optimization}
\[
\KL_{\geq c^{T - t } w^{(t+1)} } \left( {\bsy{\alpha}}^\top \mR^{(t+1)} \| \mathbf{z}^{x[:t+1]}  \right)
\]
\State Set $\boldsymbol{\alpha}^{x[:t+1]} = \{\alpha^{x[:t+1]}_{h} \}_{h \in \calH^{(t+1)}}$ to be the minimizer in the above optimization problem
\EndFor
\State Let $x$ be the string $o_1o_2 \cdots o_T$
\State \textbf{Output:} $x$
\end{algorithmic}
\label{alg:sampling-procedure}
\end{algorithm}

\begin{remark}
Note that the constraints in the optimization problem in Line~\ref{line:convex-opt} of Algorithm~\ref{alg:sampling-procedure} ensure that the KL divergence is a convex function of $\boldsymbol{\alpha}^{x[:t+1]}$ over the entire feasible domain and thus can be optimized efficiently.
\end{remark}

The analysis will rely on bounding the (truncated) KL divergence between the true distribution and the distribution that our algorithm samples from.  We will do this inductively in $t$ via a hybridization argument.  However, there are certain ``bad" histories that we will need to truncate out and these are precisely those that are not positively representable in Lemma~\ref{lem:recovery-guarantees}.
\begin{definition}\label{def:good-histories}
For $t \leq T$, we let $\calG^{(t)} \subseteq \calO^t$ be the subset of $h \in \calO^t$ such that for all $t' \leq t$, the prefix of $h$ of length $t'$ is $(2\sqrt{c}, \eta)$-positively representable by $\calH^{(t')}, \calX^{(t')}$.
\end{definition}

The key lemma of the analysis is Lemma~\ref{lem:key-closeness} but first we will need to prove a sequence of preliminary claims.  We apply Corollary~\ref{coro:KL-basic} to analyze the convex optimization step in Algorithm~\ref{alg:sampling-procedure}.  We get that, up to some small additive error, replacing $\mathbf{z}^{x[:t+1]}$ with the solution ${\bsy{\alpha}^{x[:t+1]}}^\top \mR^{(t+1)} $ that we compute reduces the KL divergence to all vectors in the feasible set of the convex program.  We can then use Assumption~\ref{assume:good-params} to lift this argument to the original distributions and argue that our change-of-basis step only incurs some small additive error overall.

\begin{claim}\label{claim:KL-projection}
Let $t < T$.  Consider the execution of Algorithm~\ref{alg:sampling-procedure} and assume that so far we have sampled $t+1$ characters i.e. $x[:t+1] = o_1\cdots o_{t+1}$.  Let $\boldsymbol{\alpha} = \{{\alpha_h} \}_{h \in \calH^{(t+1)}}$ be any feasible set of coefficients in the convex program in Line~\ref{line:convex-opt} of Algorithm~\ref{alg:sampling-procedure}.  Then we have
\[
\KL_{\geq c^{T - t} w^{(t+1)}}\left( \bsy{\alpha}^\top \mR^{(t+1)}  \| {\bsy{\alpha}^{x[:t+1]}}^\top \mR^{(t+1)}  \right) \leq \KL_{\geq  c^{T - t} w^{(t+1)}}\left( \bsy{\alpha}^\top \mR^{(t+1)} \| \mathbf{z}^{x[:t+1]} \right) + \log\left(\norm{\mathbf{z}^{x[:t+1]}}_1 + \eta \right) \,.
\]
\end{claim}
\begin{proof}
We apply Corollary~\ref{coro:KL-basic} where $\calK$ is the set of all vectors $\bsy{\alpha}^\top \mR^{(t+1)}$ where $\boldsymbol{\alpha}$ is feasible in the convex program.  It is clear that this set is convex.  Also,  if we set $w \leftarrow c^{T - t}w^{(t+1)}$ in Corollary~\ref{coro:KL-basic} then all of the hypotheses are satisfied.  By Assumption~\ref{assume:good-params}, $\norm{w}_1 \leq d\sqrt{c} \leq \eta$.  Also by definition, all elements of $\calK$ have sum of entries equal to $1$.  Thus, we get
\[
\KL_{\geq c^{T - t} w^{(t+1)}}\left(\bsy{\alpha}^\top \mR^{(t+1)}  \| {\bsy{\alpha}^{x[:t+1]}}^\top \mR^{(t+1)}  \right) \leq \KL_{\geq  c^{T - t} w^{(t+1)}}\left( \bsy{\alpha}^\top \mR^{(t+1)} \| \mathbf{z}^{x[:t+1]} \right) + \log\left(\norm{\mathbf{z}^{x[:t+1]}}_1 + \eta \right)
\]
as desired.
\end{proof}

We will also need the following observation about Algorithm~\ref{alg:sampling-procedure} that the vector 
\[
{\bsy{\alpha}^{x[:t]}}^\top \mM^{(t)}[\calH^{(t)}, :]
\]
is always close to being a valid distribution.

\begin{claim}\label{claim:close-to-distribution}
Throughout the execution of Algorithm~\ref{alg:sampling-procedure}, for any $t < T$, any feasible solution $\bsy{\alpha} = \{ \alpha_h \}_{h \in \calH^{(t+1)}}$ to the convex program in Line~\ref{line:convex-opt} must satisfy
\begin{itemize}
\item $1 - 3S^2\eta \leq \sum_{h \in \calH^{(t+1)}} \alpha_h \leq 1 + 3S^2\eta$ 
\item $1 - \eta \leq \norm{\bsy{\alpha}^\top \mM^{(t+1)}[\calH^{(t+1)} , :]}_1 \leq 1 + \eta$
\end{itemize}
\end{claim}
\begin{proof}
The third constraint of the convex program implies that
\[
1 = \sum_{h \in \calH^{(t+1)}} \alpha_h \langle u_h , \mathbf{1} \rangle = \sum_{h \in \calH^{(t+1)}}\alpha_h + \sum_{h \in \calH^{(t+1)}} \alpha_h (\langle u_h , \mathbf{1} \rangle - 1) \,.
\]
However Assumption~\ref{assume:good-params} implies that $|\langle u_h , \mathbf{1} \rangle - 1 | \leq \eta$ for all $h$ and thus we conclude 
\[
1 - 3S^2\eta \leq \sum_{h \in \calH^{(t+1)}} \alpha_h \leq 1 + 3S^2 \eta \,.
\]
Now for the second statement, note that the second and third constraints of the convex program together enforce that $\norm{\sum_{h \in \calH^{(t+1)}} \alpha_h u_h }_1 = 1$ and then Assumption~\ref{assume:good-params} implies that
\[
1 -\eta \leq \norm{\bsy{\alpha}^\top \mM^{(t+1)}[\calH^{(t+1)} , :]}_1 \leq 1 + \eta
\]
as desired.
\end{proof}

\begin{corollary}\label{coro:positive-probs}
Throughout the execution of Algorithm~\ref{alg:sampling-procedure}, the probabilities $p_o$ computed in line~\ref{line:next-char} are always at least $c^{0.1}$.     
\end{corollary}
\begin{proof}
Let the entries of ${\boldsymbol{\alpha}^{x[:t]}}^\top \mP^{(t)}$ be $\{ r_o\}_{o \in \calO}$.  By Claim~\ref{claim:close-to-distribution}, 
\[
\sum_{o \in \calO} |r_o| \leq 1 + \eta \,.
\]
Thus,
\[
\sum_{o \in \calO} \max(2c^{0.1}, r_o) \leq 1 + \eta + 2c^{0.1}O \leq 2 
\]
and this immediately implies the desired statement.
\end{proof}

Now we can take Claim~\ref{claim:KL-projection} and then use Assumption~\ref{assume:good-params} to replace the matrix $\mR^{(t+1)}$ with the matrix of true probabilities $\mM^{(t+1)}$ as follows. 

\begin{claim}\label{claim:key-closeness-p1}
Fix an $x[:t] \in \calG^{(t)}$ (recall Definition~\ref{def:good-histories}).  Let $o \in \calO$ be such that $x[:t] \vee o \in \calG^{(t+1)}$.  Assume that in the execution of Algorithm~\ref{alg:sampling-procedure}, the first $t$ characters we have sampled are exactly $x[:t]$ and the $t+1$st character is sampled to be $o_{t+1} = o$, so we just set $x[:t+1] = x[:t] \vee o$ in line~\ref{line:append}.  Then after solving the optimization problem in line~\ref{line:convex-opt} to compute the next set of coefficients $\bsy{\alpha}^{x[:t] \vee o}$, we have
\[
\begin{split}
&\KL_{\geq c^{T - t} }\left(\mM^{(t+1)}[x[:t] \vee o, :] \| {\bsy{\alpha}^{x[:t] \vee o}}^\top \mM^{(t+1)}[\calH^{(t+1)}, :] \right) \\  &\leq \KL_{\geq c^{T - t} }\left( \mM^{(t+1)}[x[:t] \vee o, :] \| {\bsy{\nu}^{x[:t] \vee o}}^\top \mM^{(t+1)}[\calH^{(t)} \vee o, :]  \right) + \log\left(\norm{\mathbf{z}^{x[:t] \vee o}}_1 + \eta \right)  + \eta^{0.5} \,.
\end{split}
\]
\end{claim}
\begin{proof}

By assumption that $x[:t] \vee o \in \calG^{(t+1)}$, there exists coefficients $\bsy{\beta}^{x[:t] \vee o} = \{ \beta_h^{x[:t] \vee o} \}_{h \in \calH^{(t+1)}}$  such that $\norm{\bsy{\beta}^{x[:t] \vee o}}_{\infty} \leq 2S$ and 
\begin{itemize}
\item \begin{equation}\label{eq:good-representation} \norm{\mM^{(t+1)}[x[:t] \vee o,:] - {\bsy{\beta}^{x[:t] \vee o}}^\top \mM^{(t+1)}[\calH^{(t+1)}, :] }_1 \leq 2S\eta \end{equation}
\item ${\bsy{\beta}^{x[:t] \vee o}}^\top  \mM^{(t+1)}[\calH^{(t+1)}, \calX^{(t+1)}]$ has all entries at least $2\sqrt{c}$ 
\end{itemize}

\noindent Note that the first condition implies that 
\[
1 - 2S\eta \leq \sum_{h \in \calH^{(t+1)}} \beta_h^{x[:t] \vee o} \leq 1 + 2S\eta
\]
because all rows of the matrix $\mM^{(t+1)}$ have sum equal to $1$.  Let $\mathbf{1}$ be the all ones vector of dimension $d$ (recall Assumption~\ref{assume:good-params}).  Then define
\[
Q \defeq \langle {\bsy{\beta}^{x[:t] \vee o}}^\top \mR^{(t+1)} ,  \mathbf{1} \rangle = \sum_{h \in \calH^{(t+1)}} \beta_h^{x[:t] \vee o} \langle u_h, \mathbf{1} \rangle = \sum_{h \in \calH^{(t+1)}} \beta_h^{x[:t] \vee o} + \sum_{h \in \calH^{(t+1)}} \beta_h^{x[:t] \vee o} \left(\langle u_h, \mathbf{1} \rangle  - 1\right) \,.
\]
By Assumption~\ref{assume:good-params} (specifically that the vectors $u_h$ are representative for the distributions $\Pr_{\wh{\Hm}}[\cdot | h]$), this implies that 
\begin{equation}\label{eq:sum-bound}
1 - 5S^2\eta  \leq Q \leq 1 + 5S^2 \eta \,.
\end{equation}
Thus, using Assumption~\ref{assume:good-params}, the vector $\bsy{\beta}^{x[:t] \vee o}/Q = \{ \beta_h^{x[:t] \vee o}/Q \}_{h \in \calH^{(t+1)}}$ must be a feasible solution to the convex program in Line~\ref{line:convex-opt} of Algorithm~\ref{alg:sampling-procedure}.    By Claim~\ref{claim:KL-projection},
\[
\begin{split}
& \KL_{\geq c^{T - t} w^{(t+1)}}\left( {\bsy{\beta}^{x[:t] \vee o}}^\top \mR^{(t+1)}/Q \| {\bsy{\alpha}^{x[:t] \vee o}}^\top \mR^{(t+1)} \right) \\ & \leq \KL_{\geq c^{T - t} w^{(t+1)}}\left( {\bsy{\beta}^{x[:t] \vee o}}^\top \mR^{(t+1)}/Q  \| \mathbf{z}^{x[:t] \vee o} \right) + \log\left(\norm{\mathbf{z}^{x[:t] \vee o}}_1 + \eta \right)
\end{split}
\]
Next, recall the vector $\bsy{\nu}^{x[:t] \vee o} \defeq  \{ q_h^{o} \alpha_h^{x[:t]}/p_{o}\}_{h \in \calH^{(t)}}$ constructed in Algorithm~\ref{alg:sampling-procedure} and recall that 
\[
\mathbf{z}^{x[:t] \vee o} = {\bsy{\nu}^{x[:t] \vee o}}^\top \mT_{o}^{(t)} \,.
\]
Then by the KL-preserving property in Assumption~\ref{assume:good-params} (and Corollary~\ref{coro:positive-probs} which bounds the entries of $\bsy{\nu}^{x[:t] \vee o}$)
\[
\begin{split}
&\KL_{\geq c^{T - t} }\left({\bsy{\beta}^{x[:t] \vee o}}^\top \mM^{(t+1)}[\calH^{(t+1)}, :]/Q \| {\bsy{\alpha}^{x[:t] \vee o}}^\top \mM^{(t+1)}[\calH^{(t+1)}, :]\right) \\  &\leq \KL_{\geq c^{T - t} }\left( {\bsy{\beta}^{x[:t] \vee o}}^\top \mM^{(t+1)}[\calH^{(t+1)}, :]/Q \| {\bsy{\nu}^{x[:t] \vee o}}^\top \mM^{(t+1)}[\calH^{(t)} \vee o, :] \right) + \log\left(\norm{\mathbf{z}^{x[:t] \vee o}}_1 + \eta \right) + 2\eta \,.
\end{split} \,.
\]
Finally, by Claim~\ref{claim:trunc-KL-lipchitz} and \eqref{eq:good-representation}, \eqref{eq:sum-bound}, the above implies
\begin{equation}\label{eq:KL-projection-bound}
\begin{split}
&\KL_{\geq c^{T - t} }\left(\mM^{(t+1)}[x[:t] \vee o, :] \| {\bsy{\alpha}^{x[:t] \vee o}}^\top \mM^{(t+1)}[\calH^{(t+1)}, :] \right) \\  &\leq \KL_{\geq c^{T - t} }\left( \mM^{(t+1)}[x[:t] \vee o, :] \| {\bsy{\nu}^{x[:t] \vee o}}^\top \mM^{(t+1)}[\calH^{(t)} \vee o, :]  \right) + \log\left(\norm{\mathbf{z}^{x[:t] \vee o}}_1 + \eta \right)  + \eta^{0.5}
\end{split} 
\end{equation} 
where in the last step we used the assumption that $\eta$ is sufficiently small.  This completes the proof.
\end{proof}

Now we sum Claim~\ref{claim:key-closeness-p1} over all possibilities for $o \in \calO$ to get the following bound.

\begin{claim}\label{claim:key-closeness-p2}
Fix an $x[:t] \in \calG^{(t)}$.  Assume that in the execution of Algorithm~\ref{alg:sampling-procedure}, the first $t$ characters we have sampled are exactly $x[:t]$.  Then we have
\[
\begin{split}
&\sum_{o \in \calO} \Pr_{\wh{\Hm}}[o | x[:t]] \cdot  \KL_{\geq c^{T - t} }\left(\mM^{(t+1)}[x[:t] \vee o, :] \| {\bsy{\alpha}^{x[:t] \vee o}}^\top \mM^{(t+1)}[\calH^{(t+1)}, :] \right) \\ & \leq \KL_{\geq c^{T - t + 1} }\left( \mM^{(t)}[x[:t] , :] \| {\bsy{\alpha}^{x[:t] }}^\top \mM^{(t)}[\calH^{(t)}, :]  \right) + 2\eta^{0.5}  + 5T\log(1/c) \cdot \sum_{\substack{o \in \calO \\ x[:t] \vee o \notin \calG^{(t+1)}}} \Pr_{\wh{\Hm}}[o|x[:t]] \,.
\end{split}
\]
\end{claim}
\begin{proof}
Let $\calF_{o}^{T - t} \subset  \calO^{T - t}$ be the subset of all strings whose first character is $o$.  Then we have 
\[
{\bsy{\nu}^{x[:t] \vee o}}^\top \mM^{(t+1)}[\calH^{(t)} \vee o, :]   = \frac{1}{p_o} \sum_{h \in \calH^{(t)}} q_h^o \alpha_{h}^{x[:t]} \mM^{(t+1)}[h \vee o, :] = \frac{1}{p_o} {\bsy{\alpha}^{x[:t]}}^\top \mM^{(t)}[\calH^{(t)}, \calF_{o}^{T - t}] \,.
\]
Also note that 
\[
\mM^{(t+1)}[x[:t] \vee o, :] =  \frac{ \mM^{(t)}[x[:t] , \calF_o^{T - t}] }{\Pr_{\wh{\Hm}}[o | x[:t]]}\,.
\]
Note that all entries of the matrix $\mM^{(t)}$ are at least $c^{T - t}$ by Assumption~\ref{assume:positive-distribution} and  $p_o \geq c^{0.1}$ by Corollary~\ref{coro:positive-probs} so using the definition of truncated KL,
\begin{equation}\label{eq:kl-transition}
\begin{split}
&\Pr_{\wh{\Hm}}[o | x[:t]]  \cdot \KL_{\geq c^{T - t} }\left( \mM^{(t+1)}[x[:t] \vee o, :] \| {\bsy{\nu}^{x[:t] \vee o}}^\top \mM^{(t+1)}[\calH^{(t)} \vee o, :]  \right) \\ &\leq  \KL_{\geq c^{T - t + 1} }\left( \mM^{(t)}[x[:t] , \calF_{o}^{T - t}] \| {\bsy{\alpha}^{x[:t] }}^\top \mM^{(t)}[\calH^{(t)}, \calF_{o}^{T - t}]  \right) + \Pr_{\wh{\Hm}}[o | x[:t]] \log \frac{p_o}{\Pr_{\wh{\Hm}}[o | x[:t]]} \,.
\end{split}
\end{equation}

Next, we will apply Claim~\ref{claim:key-closeness-p1} and \eqref{eq:kl-transition} to bound the sum 
\[
\sum_{o \in \calO} \Pr_{\wh{\Hm}}[o | x[:t]] \cdot  \KL_{\geq c^{T - t} }\left(\mM^{(t+1)}[x[:t] \vee o, :] \| {\bsy{\alpha}^{x[:t] \vee o}}^\top \mM^{(t+1)}[\calH^{(t+1)}, :] \right) \,.
\]
However, we can only apply Claim~\ref{claim:key-closeness-p1} when $o$ is such that $x[:t] \vee o \in \calG^{(t+1)}$.  For other choices of $o$, we can nevertheless use the following trivial version of Claim~\ref{claim:key-closeness-p1}
\[
\begin{split}
&\KL_{\geq c^{T - t} }\left(\mM^{(t+1)}[x[:t] \vee o, :] \| {\bsy{\alpha}^{x[:t] \vee o}}^\top \mM^{(t+1)}[\calH^{(t+1)}, :] \right) \\  &\leq \KL_{\geq c^{T - t} }\left( \mM^{(t+1)}[x[:t] \vee o, :] \| {\bsy{\nu}^{x[:t] \vee o}}^\top \mM^{(t+1)}[\calH^{(t)} \vee o, :]  \right) + \log\left(\norm{\mathbf{z}^{x[:t] \vee o}}_1 + \eta \right)  + \eta^{0.5} + 5T \log(1/c) 
\end{split}
\]
where the above holds simply because both of the KL divergences are bounded in magnitude by $2T \log(1/c)$.  Thus, we have the bound
\begin{equation}\label{eq:main-bound}
\begin{split}
& \sum_{o \in \calO} \Pr_{\wh{\Hm}}[o | x[:t]] \cdot  \KL_{\geq c^{T - t} }\left(\mM^{(t+1)}[x[:t] \vee o, :] \| {\bsy{\alpha}^{x[:t] \vee o}}^\top \mM^{(t+1)}[\calH^{(t+1)}, :] \right) \\ & \leq \eta^{0.5} + \sum_{o \in \calO} \Pr_{\wh{\Hm}}[o | x[:t]] \cdot  \log\left( \norm{\mathbf{z}^{x[:t] \vee o}}_1 + \eta \right)+ \sum_{o \in \calO } \Pr_{\wh{\Hm}}[o | x[:t]] \log \frac{p_o}{\Pr_{\wh{\Hm}}[o | x[:t]]}  \\ & \quad  + \sum_{o \in \calO} \KL_{\geq c^{T - t + 1} }\left( \mM^{(t)}[x[:t] , \calF_{o}^{T - t}] \| {\bsy{\alpha}^{x[:t] }}^\top \mM^{(t)}[\calH^{(t)}, \calF_{o}^{T - t}]  \right) + 5T\log(1/c) \cdot \sum_{\substack{o \in \calO \\ x[:t] \vee o \notin \calG^{(t+1)}}} \Pr_{\wh{\Hm}}[o|x[:t]]
\\ & = \KL_{\geq c^{T - t + 1} }\left( \mM^{(t)}[x[:t] , :] \| {\bsy{\alpha}^{x[:t] }}^\top \mM^{(t)}[\calH^{(t)}, :]  \right) \\ & \quad + \eta^{0.5} + \sum_{o \in \calO} \Pr_{\wh{\Hm}}[o | x[:t]] \cdot  \log\left( \frac{\left(\norm{\mathbf{z}^{x[:t] \vee o}}_1 + \eta \right) p_o}{\Pr_{\wh{\Hm}}[o | x[:t]]}\right)  +  5T\log(1/c) \cdot \sum_{\substack{o \in \calO \\ x[:t] \vee o \notin \calG^{(t+1)}}} \Pr_{\wh{\Hm}}[o|x[:t]] \,.
\end{split}
\end{equation}

Next, by definition and Assumption~\ref{assume:good-params},
\[
p_o \norm{ \mathbf{z}^{x[:t] \vee o}}_1 = \norm{p_o {\bsy{\nu}^{x[:t]\vee o}}^\top \mT_o^{(t)} }_1 \leq \norm{ p_o{\bsy{\nu}^{x[:t]\vee o}}^\top \mM^{(t+1)}[ \calH^{(t)} \vee o, :] }_1 + \eta \,.
\]
Observe that the vector ${\bsy{\alpha}^{x[:t]}}^\top \mM^{(t)}[\calH^{(t)}, :]$ is the same as concatenating $p_o{\bsy{\nu}^{x[:t]\vee o}}^\top \mM^{(t+1)}[ \calH^{(t)} \vee o, :]$ over all choices of $o$.  Thus by Claim~\ref{claim:close-to-distribution}
\[
\sum_{o \in \calO} p_o \left(\norm{ \mathbf{z}^{x[:t] \vee o}}_1  + \eta\right) \leq  (O+1)\eta  + \norm{{\bsy{\alpha}^{x[:t]}}^\top \mM^{(t)}[\calH^{(t)}, :] }_1 \leq 1 + (O + 2)\eta  \,.
\]
Combining this with \eqref{eq:main-bound} and using nonnegativity of KL gives
\[
\begin{split}
&\sum_{o \in \calO} \Pr_{\wh{\Hm}}[o | x[:t]] \cdot  \KL_{\geq c^{T - t} }\left(\mM^{(t+1)}[x[:t] \vee o, :] \| {\bsy{\alpha}^{x[:t] \vee o}}^\top \mM^{(t+1)}[\calH^{(t+1)}, :] \right) \\ & \leq \KL_{\geq c^{T - t + 1} }\left( \mM^{(t)}[x[:t] , :] \| {\bsy{\alpha}^{x[:t] }}^\top \mM^{(t)}[\calH^{(t)}, :]  \right) + \eta^{0.5} + \log\left(\sum_{o \in \calO} p_o \left(\norm{ \mathbf{z}^{x[:t] \vee o}}_1  + \eta\right) \right)\\ &\quad + 5T\log(1/c) \cdot \sum_{\substack{o \in \calO \\ x[:t] \vee o \notin \calG^{(t+1)}}} \Pr_{\wh{\Hm}}[o|x[:t]] \\ & \leq \KL_{\geq c^{T - t + 1} }\left( \mM^{(t)}[x[:t] , :] \| {\bsy{\alpha}^{x[:t] }}^\top \mM^{(t)}[\calH^{(t)}, :]  \right) + 2\eta^{0.5} + 5T\log(1/c) \cdot \sum_{\substack{o \in \calO \\ x[:t] \vee o \notin \calG^{(t+1)}}} \Pr_{\wh{\Hm}}[o|x[:t]] \,.
\end{split}
\]

\end{proof}

Next, we sum Claim~\ref{claim:key-closeness-p2} over all possibilities for $x[:t] \in \calG^{(t)}$.  Note that in the execution of Algorithm~\ref{alg:sampling-procedure}, we can associate a unique vector of coefficients $\bsy{\alpha}^{h}$ to every possible history $h$ to be the state of the vector $\bsy{\alpha}^{x[:t]}$ conditioned on the prefix $x[:t]$ being equal to $h$.

\begin{lemma}\label{lem:key-closeness}
For all $0 \leq t \leq T - 1$, we have
\[
\begin{split}
&\sum_{h \in \calG^{(t+1)}}\Pr_{\wh{\Hm}}[h \vee o] \cdot  \KL_{\geq c^{T - t} }\left(\mM^{(t+1)}[h \vee o, :] \| {\bsy{\alpha}^{h \vee o}}^\top \mM^{(t+1)}[\calH^{(t+1)}, :] \right) \\ & \leq \sum_{h \in \calG^{(t)}} \Pr_{\wh{\Hm}}[h] \cdot \KL_{\geq c^{T - t + 1} }\left( \mM^{(t)}[h  , :] \| {\bsy{\alpha}^{h}}^\top \mM^{(t)}[\calH^{(t)}, :]  \right) + 3\eta^{0.5}  \,.
\end{split}
\]
\end{lemma}
\begin{proof}
Summing Claim~\ref{claim:key-closeness-p2} as $x[:t]$ ranges over all $h \in \calG^{(t)}$ (with weights $\Pr_{\wh{\Hm}}[h]$) gives
\begin{equation}\label{eq:almost-final}
\begin{split}
&\sum_{h \in \calG^{(t)}, o \in \calO}\Pr_{\wh{\Hm}}[h \vee o] \cdot  \KL_{\geq c^{T - t} }\left(\mM^{(t+1)}[h \vee o, :] \| {\bsy{\alpha}^{h \vee o}}^\top \mM^{(t+1)}[\calH^{(t+1)}, :] \right) \\ & \leq 5T\log(1/c) \cdot \sum_{h \vee o \notin \calG^{(t+1)}}\Pr_{\wh{\Hm}}[h \vee o]  + 2\eta^{0.5} + \sum_{h \in \calG^{(t)}} \Pr_{\wh{\Hm}}[h] \cdot \KL_{\geq c^{T - t + 1} }\left( \mM^{(t)}[h  , :] \| {\bsy{\alpha}^{h}}^\top \mM^{(t)}[\calH^{(t)}, :]  \right)   \,.
\end{split}
\end{equation}

Finally, note that 
\[
\left\lvert \KL_{\geq c^{T - t} }\left(\mM^{(t+1)}[h \vee o, :] \| {\bsy{\alpha}^{h \vee o}}^\top \mM^{(t+1)}[\calH^{(t+1)}, :] \right) \right\rvert \leq 2T \log (1/c)
\]
because $\mM^{(t+1)}[h \vee o, :]$ has nonnegative entries summing to $1$ and Claim~\ref{claim:close-to-distribution} bounds the entries of the second part.  Thus by Assumption~\ref{assume:good-params} (specifically the condition about $h \sim \wh{\Hm}[:t]$ being positively representable), we conclude
\[
\sum_{h \vee o \notin \calG^{(t+1)}}\Pr_{\wh{\Hm}}[h \vee o]  \leq T\eta   \,.
\]
Thus from \eqref{eq:almost-final} and using our assumptions on $\eta$, we get
\[
\begin{split}
&\sum_{h \in \calG^{(t+1)}}\Pr_{\wh{\Hm}}[h \vee o] \cdot  \KL_{\geq c^{T - t} }\left(\mM^{(t+1)}[h \vee o, :] \| {\bsy{\alpha}^{h \vee o}}^\top \mM^{(t+1)}[\calH^{(t+1)}, :] \right) \\ & \leq \sum_{h \in \calG^{(t)}} \Pr_{\wh{\Hm}}[h] \cdot \KL_{\geq c^{T - t + 1} }\left( \mM^{(t)}[h  , :] \| {\bsy{\alpha}^{h}}^\top \mM^{(t)}[\calH^{(t)}, :]  \right) + 3\eta^{0.5}  
\end{split}
\]
as desired.
\end{proof}

Now we can prove the main lemma of this section, where we show that the distribution that our algorithm samples from is close to the desired distribution $\wh{\Hm}$.

\begin{lemma}\label{lem:sampling-is-good}
Under Assumption~\ref{assume:good-params} and Assumption~\ref{assume:positive-distribution}, the distribution that Algorithm~\ref{alg:sampling-procedure} samples from, say $\Hm'$, satisfies
\[
d_{\TV}( \wh{\Hm}, \Hm') \leq 10 T \eta^{0.1} \,.
\]
\end{lemma}
\begin{proof}
First, for a fixed history $h \in \calO^{(t)}$, let $\bsy{\beta}^h$ be the vector 
\[
\max( {\bsy{\alpha}^{h}}^\top \mM^{(t)}[\calH^{(t)}, :] , c^{T - t + 1})
\]
where the maximum is taken entrywise.  Clearly $\bsy{\beta}^h$ has all positive entries and by Claim~\ref{claim:close-to-distribution} (and the definition of $c$), the sum of the entries is at most $1 + 2\eta$.  Now let 
\[
\bsy{\theta}^h = \frac{\bsy{\beta}^h}{\norm{\bsy{\beta}^h}_1}
\]
i.e. normalizing so that the sum of the entries is $1$.  Note that all of the entries of the matrix $\mM^{(t)}$ are at least $c^{T - t}$ so
\[
\KL_{\geq c^{T - t + 1} }\left( \mM^{(t)}[h  , :] \| {\bsy{\alpha}^{h}}^\top \mM^{(t)}[\calH^{(t)}, :]  \right)  = \KL\left( \mM^{(t)}[h  , :] \| \bsy{\beta}^h \right)  \geq \KL\left( \mM^{(t)}[h  , :] \| \bsy{\theta}^h \right) - 2\eta \,.
\]
Now combining the above with Lemma~\ref{lem:key-closeness} implies
\[
\sum_{h \in \calG^{(t)}} \Pr_{\wh{\Hm}}[h] \cdot \KL\left( \mM^{(t)}[h  , :] \| \bsy{\theta}^h  \right) \leq 4t\eta^{0.5} \leq \eta^{0.4} \,.
\]
Recall by Assumption~\ref{assume:good-params} that 
\[
\Pr_{h \sim \wh{\Hm}[:t]} [h \in \calG^{(t)}] \geq 1 - \eta
\]
and thus by Pinsker's inequality, we deduce 
\begin{equation}\label{eq:good-event}
\Pr_{h \sim \wh{\Hm}[:t]}\left[\norm{\bsy{\theta}^h - \mM^{(t)}[h  , :]}_1 \geq \eta^{0.1} \right] \leq 2\eta^{0.2} 
\end{equation}
and when this happens, since 
\[
\norm{ {\bsy{\alpha}^{h}}^\top \mM^{(t)}[\calH^{(t)}, :] - \bsy{\theta}^h} \leq 4\eta 
\]
we get that the distributions 
$\{\Pr_{\wh{\Hm}}[o|h] \}_{o \in \calO} , \{\Pr_{\Hm'}[o | h]\}_{o \in \calO}$ are $2\eta^{0.1}$-close in TV distance.  Now we construct a hybrid distribution $\Hm''$ that is equal to $\Hm'$ except starting from every prefix $h$ where the condition in \eqref{eq:good-event} fails, all future characters are sampled according to $\wh{\Hm}$.  By \eqref{eq:good-event}, 
$d_{\TV}(\Hm'', \Hm') \leq 2T\eta^{0.2}$ and by Claim~\ref{claim:nextchar-to-TV}, $d_{\TV}(\Hm'', \wh{\Hm}) \leq 2T\eta^{0.1}$.  Putting these together yields the desired statement.

\end{proof}

Now we can put everything together and complete the proof of Theorem~\ref{thm:full-learning}.

\begin{proof}[Proof of Theorem~\ref{thm:full-learning}]
Let $\eps = 1/\poly(S,O,T, 1/\eta)$ for some sufficiently large polynomial.  Now we will combine the following ingredients 
\begin{itemize}
    \item Lemma~\ref{lem:pdf-estimation} to get exact sample and pdf access to a distribution $\wh{\Hm}$ close to $\Hm$
    \item Algorithm~\ref{alg:learning} to learn a description of a distribution 
    \item Algorithm~\ref{alg:sampling-procedure} to draw samples from the distribution parameterized by this learned description
\end{itemize} 
For all of these algorithms, we will set the failure probability parameter $\delta \leftarrow 0.1\eta$.  First, we apply Lemma~\ref{lem:pdf-estimation} to get exact sample and pdf access to a distribution $\wh{\Hm}$ that is $\eps$ conditionally close to $\Hm$ and $\eps/(10O)^2$-positive.  Next, we apply Algorithm~\ref{alg:learning} to learn a set of parameters
\[
\{\mP^{(t)} , \calH^{(t)}, \calX^{(t)}, \{ u_h\}_{h \in \calB^{(t)}}, w^{(t)} \}_{t \in [T]} 
\]
where $\calB^{(t)} = \calH^{(t)} \cup \{ \calH^{(t-1)} \vee o\}_{o \in \calO}$.  We apply Lemma~\ref{lem:recovery-guarantees} and Claim~\ref{claim:ref-entries-small} on these learned parameters.  They, combined with the definition of the vectors $w^{(t)}, \{ u_h^{(t)} \}_{h \in \calB^{(t)} }$, imply that with probability $1 - \eta$, the conditions of Assumption~\ref{assume:good-params}  are satisfied.  Also, Assumption~\ref{assume:positive-distribution} is satisfied by construction.  Finally Lemma~\ref{lem:sampling-is-good} implies that when these conditions are satisfied, the distribution that Algorithm~\ref{alg:sampling-procedure} samples from is $10T\eta^{0.1}$-close to $\wh{\Hm}$ in TV distance.  Note that the sampling algorithm runs in time $\poly(S,O,T \log(1/\eta))$.  Redefining $\eta \leftarrow (0.1\eta/T)^{10}$ completes the proof.  
\end{proof}

\bibliographystyle{alpha}
\bibliography{bibliography}

\newcommand{\etalchar}[1]{$^{#1}$}
\begin{thebibliography}{BLMY23b}

\bibitem[AHK12]{anandkumar2012method}
Animashree Anandkumar, Daniel Hsu, and Sham~M Kakade.
\newblock A method of moments for mixture models and hidden markov models.
\newblock In {\em Conference on learning theory}, pages 33--1. JMLR Workshop and Conference Proceedings, 2012.

\bibitem[AK08]{awerbuch2008online}
Baruch Awerbuch and Robert Kleinberg.
\newblock Online linear optimization and adaptive routing.
\newblock {\em Journal of Computer and System Sciences}, 74(1):97--114, 2008.

\bibitem[Ang87]{angluin1987learning}
Dana Angluin.
\newblock Learning regular sets from queries and counterexamples.
\newblock {\em Information and computation}, 75(2):87--106, 1987.

\bibitem[BC18]{bhattacharyya2018property}
Rishiraj Bhattacharyya and Sourav Chakraborty.
\newblock Property testing of joint distributions using conditional samples.
\newblock {\em ACM Transactions on Computation Theory (TOCT)}, 10(4):1--20, 2018.

\bibitem[BCPV19]{bhaskara2019smoothed}
Aditya Bhaskara, Aidao Chen, Aidan Perreault, and Aravindan Vijayaraghavan.
\newblock Smoothed analysis in unsupervised learning via decoupling.
\newblock In {\em 2019 IEEE 60th Annual Symposium on Foundations of Computer Science (FOCS)}, pages 582--610. IEEE, 2019.

\bibitem[BLMY23a]{bakshi2023new}
Ainesh Bakshi, Allen Liu, Ankur Moitra, and Morris Yau.
\newblock A new approach to learning linear dynamical systems.
\newblock In {\em Proceedings of the 55th Annual ACM Symposium on Theory of Computing}, pages 335--348, 2023.

\bibitem[BLMY23b]{bakshi2023tensor}
Ainesh Bakshi, Allen Liu, Ankur Moitra, and Morris Yau.
\newblock Tensor decompositions meet control theory: learning general mixtures of linear dynamical systems.
\newblock In {\em International Conference on Machine Learning}, pages 1549--1563. PMLR, 2023.

\bibitem[CFGM13]{chakraborty2013power}
Sourav Chakraborty, Eldar Fischer, Yonatan Goldhirsh, and Arie Matsliah.
\newblock On the power of conditional samples in distribution testing.
\newblock In {\em Proceedings of the 4th conference on Innovations in Theoretical Computer Science}, pages 561--580, 2013.

\bibitem[CGG01]{cryan2001evolutionary}
Mary Cryan, Leslie~Ann Goldberg, and Paul~W Goldberg.
\newblock Evolutionary trees can be learned in polynomial time in the two-state general markov model.
\newblock {\em SIAM Journal on Computing}, 31(2):375--397, 2001.

\bibitem[CJLW21]{chen2021learning}
Xi~Chen, Rajesh Jayaram, Amit Levi, and Erik Waingarten.
\newblock Learning and testing junta distributions with sub cube conditioning.
\newblock In {\em Conference on Learning Theory}, pages 1060--1113. PMLR, 2021.

\bibitem[CP22]{chen2022learning}
Yanxi Chen and H~Vincent Poor.
\newblock Learning mixtures of linear dynamical systems.
\newblock In {\em International conference on machine learning}, pages 3507--3557. PMLR, 2022.

\bibitem[CPD{\etalchar{+}}24]{carlini2024stealing}
Nicholas Carlini, Daniel Paleka, Krishnamurthy~Dj Dvijotham, Thomas Steinke, Jonathan Hayase, A~Feder Cooper, Katherine Lee, Matthew Jagielski, Milad Nasr, Arthur Conmy, et~al.
\newblock Stealing part of a production language model.
\newblock {\em arXiv preprint arXiv:2403.06634}, 2024.

\bibitem[CRS15]{canonne2015testing}
Cl{\'e}ment~L Canonne, Dana Ron, and Rocco~A Servedio.
\newblock Testing probability distributions using conditional samples.
\newblock {\em SIAM Journal on Computing}, 44(3):540--616, 2015.

\bibitem[DKKZ20]{diakonikolas2020algorithms}
Ilias Diakonikolas, Daniel~M Kane, Vasilis Kontonis, and Nikos Zarifis.
\newblock Algorithms and sq lower bounds for pac learning one-hidden-layer relu networks.
\newblock In {\em Conference on Learning Theory}, pages 1514--1539. PMLR, 2020.

\bibitem[FKQR21]{foster2021statistical}
Dylan~J Foster, Sham~M Kakade, Jian Qian, and Alexander Rakhlin.
\newblock The statistical complexity of interactive decision making.
\newblock {\em arXiv preprint arXiv:2112.13487}, 2021.

\bibitem[GGJ{\etalchar{+}}20]{goel2020superpolynomial}
Surbhi Goel, Aravind Gollakota, Zhihan Jin, Sushrut Karmalkar, and Adam Klivans.
\newblock Superpolynomial lower bounds for learning one-layer neural networks using gradient descent.
\newblock In {\em International Conference on Machine Learning}, pages 3587--3596. PMLR, 2020.

\bibitem[GGK20]{goel2020statistical}
Surbhi Goel, Aravind Gollakota, and Adam Klivans.
\newblock Statistical-query lower bounds via functional gradients.
\newblock {\em Advances in Neural Information Processing Systems}, 33:2147--2158, 2020.

\bibitem[GMR23]{golowich2023planning}
Noah Golowich, Ankur Moitra, and Dhruv Rohatgi.
\newblock Planning and learning in partially observable systems via filter stability.
\newblock In {\em Proceedings of the 55th Annual ACM Symposium on Theory of Computing}, pages 349--362, 2023.

\bibitem[HGKD15]{huang2015minimal}
Qingqing Huang, Rong Ge, Sham Kakade, and Munther Dahleh.
\newblock Minimal realization problems for hidden markov models.
\newblock {\em IEEE Transactions on Signal Processing}, 64(7):1896--1904, 2015.

\bibitem[HJB{\etalchar{+}}21]{he2021stealing}
Xinlei He, Jinyuan Jia, Michael Backes, Neil~Zhenqiang Gong, and Yang Zhang.
\newblock Stealing links from graph neural networks.
\newblock In {\em 30th USENIX security symposium (USENIX security 21)}, pages 2669--2686, 2021.

\bibitem[HKZ12]{hsu2012spectral}
Daniel Hsu, Sham~M Kakade, and Tong Zhang.
\newblock A spectral algorithm for learning hidden markov models.
\newblock {\em Journal of Computer and System Sciences}, 78(5):1460--1480, 2012.

\bibitem[HLXS21]{he2021model}
Xuanli He, Lingjuan Lyu, Qiongkai Xu, and Lichao Sun.
\newblock Model extraction and adversarial transferability, your bert is vulnerable!
\newblock {\em arXiv preprint arXiv:2103.10013}, 2021.

\bibitem[HMR18]{hardt2018gradient}
Moritz Hardt, Tengyu Ma, and Benjamin Recht.
\newblock Gradient descent learns linear dynamical systems.
\newblock {\em Journal of Machine Learning Research}, 19(29):1--44, 2018.

\bibitem[Jac97]{jackson1997efficient}
Jeffrey~C Jackson.
\newblock An efficient membership-query algorithm for learning dnf with respect to the uniform distribution.
\newblock {\em Journal of Computer and System Sciences}, 55(3):414--440, 1997.

\bibitem[JCCCP21]{jia2021entangled}
Hengrui Jia, Christopher~A Choquette-Choo, Varun Chandrasekaran, and Nicolas Papernot.
\newblock Entangled watermarks as a defense against model extraction.
\newblock In {\em 30th USENIX security symposium (USENIX Security 21)}, pages 1937--1954, 2021.

\bibitem[JGH18]{jacot2018neural}
Arthur Jacot, Franck Gabriel, and Cl{\'e}ment Hongler.
\newblock Neural tangent kernel: Convergence and generalization in neural networks.
\newblock {\em Advances in neural information processing systems}, 31, 2018.

\bibitem[JKA{\etalchar{+}}17]{jiang2017contextual}
Nan Jiang, Akshay Krishnamurthy, Alekh Agarwal, John Langford, and Robert~E Schapire.
\newblock Contextual decision processes with low bellman rank are pac-learnable.
\newblock In {\em International Conference on Machine Learning}, pages 1704--1713. PMLR, 2017.

\bibitem[JSMA19]{juuti2019prada}
Mika Juuti, Sebastian Szyller, Samuel Marchal, and N~Asokan.
\newblock Prada: protecting against dnn model stealing attacks.
\newblock In {\em 2019 IEEE European Symposium on Security and Privacy (EuroS\&P)}, pages 512--527. IEEE, 2019.

\bibitem[KKMZ24]{kakade2024learning}
Sham~M. Kakade, Akshay Krishnamurthy, Gaurav Mahajan, and Cyril Zhang.
\newblock Learning hidden markov models using conditional samples, 2024.

\bibitem[KNW13]{kontorovich2013learning}
Aryeh Kontorovich, Boaz Nadler, and Roi Weiss.
\newblock On learning parametric-output hmms.
\newblock In {\em International Conference on Machine Learning}, pages 702--710. PMLR, 2013.

\bibitem[KS09]{klivans2009cryptographic}
Adam~R Klivans and Alexander~A Sherstov.
\newblock Cryptographic hardness for learning intersections of halfspaces.
\newblock {\em Journal of Computer and System Sciences}, 75(1):2--12, 2009.

\bibitem[LZJ{\etalchar{+}}22]{li2022defending}
Yiming Li, Linghui Zhu, Xiaojun Jia, Yong Jiang, Shu-Tao Xia, and Xiaochun Cao.
\newblock Defending against model stealing via verifying embedded external features.
\newblock In {\em Proceedings of the AAAI conference on artificial intelligence}, volume~36, pages 1464--1472, 2022.

\bibitem[MCH{\etalchar{+}}21]{ma2021undistillable}
Haoyu Ma, Tianlong Chen, Ting-Kuei Hu, Chenyu You, Xiaohui Xie, and Zhangyang Wang.
\newblock Undistillable: Making a nasty teacher that cannot teach students.
\newblock {\em arXiv preprint arXiv:2105.07381}, 2021.

\bibitem[MMM19]{mei2019mean}
Song Mei, Theodor Misiakiewicz, and Andrea Montanari.
\newblock Mean-field theory of two-layers neural networks: dimension-free bounds and kernel limit.
\newblock In {\em Conference on learning theory}, pages 2388--2464. PMLR, 2019.

\bibitem[MR05]{mossel2005learning}
Elchanan Mossel and S{\'e}bastien Roch.
\newblock Learning nonsingular phylogenies and hidden markov models.
\newblock In {\em Proceedings of the thirty-seventh annual ACM symposium on Theory of computing}, pages 366--375, 2005.

\bibitem[NKIH23]{naseh2023stealing}
Ali Naseh, Kalpesh Krishna, Mohit Iyyer, and Amir Houmansadr.
\newblock Stealing the decoding algorithms of language models.
\newblock In {\em Proceedings of the 2023 ACM SIGSAC Conference on Computer and Communications Security}, pages 1835--1849, 2023.

\bibitem[NP23]{nguyen2023rigorous}
Phan-Minh Nguyen and Huy~Tuan Pham.
\newblock A rigorous framework for the mean field limit of multilayer neural networks.
\newblock {\em Mathematical Statistics and Learning}, 6(3):201--357, 2023.

\bibitem[OMR23]{oliynyk2023know}
Daryna Oliynyk, Rudolf Mayer, and Andreas Rauber.
\newblock I know what you trained last summer: A survey on stealing machine learning models and defences.
\newblock {\em ACM Computing Surveys}, 55(14s):1--41, 2023.

\bibitem[OO19]{oymak2019non}
Samet Oymak and Necmiye Ozay.
\newblock Non-asymptotic identification of lti systems from a single trajectory.
\newblock In {\em 2019 American control conference (ACC)}, pages 5655--5661. IEEE, 2019.

\bibitem[OSF19]{orekondy2019prediction}
Tribhuvanesh Orekondy, Bernt Schiele, and Mario Fritz.
\newblock Prediction poisoning: Towards defenses against dnn model stealing attacks.
\newblock {\em arXiv preprint arXiv:1906.10908}, 2019.

\bibitem[RST19]{reith2019efficiently}
Robert~Nikolai Reith, Thomas Schneider, and Oleksandr Tkachenko.
\newblock Efficiently stealing your machine learning models.
\newblock In {\em Proceedings of the 18th ACM Workshop on Privacy in the Electronic Society}, pages 198--210, 2019.

\bibitem[SBR19]{simchowitz2019learning}
Max Simchowitz, Ross Boczar, and Benjamin Recht.
\newblock Learning linear dynamical systems with semi-parametric least squares.
\newblock In {\em Conference on Learning Theory}, pages 2714--2802. PMLR, 2019.

\bibitem[Sha51]{shannon1951prediction}
Claude~E Shannon.
\newblock Prediction and entropy of printed english.
\newblock {\em Bell system technical journal}, 30(1):50--64, 1951.

\bibitem[SKLV17]{sharan2017learning}
Vatsal Sharan, Sham~M Kakade, Percy~S Liang, and Gregory Valiant.
\newblock Learning overcomplete hmms.
\newblock {\em Advances in Neural Information Processing Systems}, 30, 2017.

\bibitem[SKLV18]{sharan2018prediction}
Vatsal Sharan, Sham Kakade, Percy Liang, and Gregory Valiant.
\newblock Prediction with a short memory.
\newblock In {\em Proceedings of the 50th Annual ACM SIGACT Symposium on Theory of Computing}, pages 1074--1087, 2018.

\bibitem[SRD19]{sarkar2019nonparametric}
Tuhin Sarkar, Alexander Rakhlin, and Munther~A Dahleh.
\newblock Nonparametric finite time lti system identification.
\newblock {\em arXiv preprint arXiv:1902.01848}, 2019.

\bibitem[SZ24]{sha2024prompt}
Zeyang Sha and Yang Zhang.
\newblock Prompt stealing attacks against large language models.
\newblock {\em arXiv preprint arXiv:2402.12959}, 2024.

\bibitem[TP19]{tsiamis2019finite}
Anastasios Tsiamis and George~J Pappas.
\newblock Finite sample analysis of stochastic system identification.
\newblock In {\em 2019 IEEE 58th Conference on Decision and Control (CDC)}, pages 3648--3654. IEEE, 2019.

\bibitem[TZJ{\etalchar{+}}16]{tramer2016stealing}
Florian Tram{\`e}r, Fan Zhang, Ari Juels, Michael~K Reiter, and Thomas Ristenpart.
\newblock Stealing machine learning models via prediction $\{$APIs$\}$.
\newblock In {\em 25th USENIX security symposium (USENIX Security 16)}, pages 601--618, 2016.

\bibitem[VSP{\etalchar{+}}17]{attention}
Ashish Vaswani, Noam Shazeer, Niki Parmar, Jakob Uszkoreit, Llion Jones, Aidan~N Gomez, \L~ukasz Kaiser, and Illia Polosukhin.
\newblock Attention is all you need.
\newblock In I.~Guyon, U.~Von Luxburg, S.~Bengio, H.~Wallach, R.~Fergus, S.~Vishwanathan, and R.~Garnett, editors, {\em Advances in Neural Information Processing Systems}, volume~30. Curran Associates, Inc., 2017.

\bibitem[WG18]{wang2018stealing}
Binghui Wang and Neil~Zhenqiang Gong.
\newblock Stealing hyperparameters in machine learning.
\newblock In {\em 2018 IEEE symposium on security and privacy (SP)}, pages 36--52. IEEE, 2018.

\bibitem[WXGD20]{wang2020information}
Xinran Wang, Yu~Xiang, Jun Gao, and Jie Ding.
\newblock Information laundering for model privacy.
\newblock {\em arXiv preprint arXiv:2009.06112}, 2020.

\bibitem[ZWL23]{zhao2023protecting}
Xuandong Zhao, Yu-Xiang Wang, and Lei Li.
\newblock Protecting language generation models via invisible watermarking.
\newblock In {\em International Conference on Machine Learning}, pages 42187--42199. PMLR, 2023.

\end{thebibliography}

\end{document}